\documentclass{article}

\usepackage{microtype}
\usepackage{graphicx}
\usepackage{subcaption}
\usepackage{dsfont}
\usepackage{booktabs} 

\usepackage{hyperref}

\usepackage{amsmath}
\usepackage{amsfonts}
\usepackage{bm}
\usepackage{amsthm}
\usepackage{cleveref}
\usepackage{color}
\usepackage{wrapfig}

\newcommand{\bx}{\mathbf{x}}
\newcommand{\bk}{\mathbf{k}}
\newcommand{\by}{\mathbf{y}}
\newcommand{\bz}{\mathbf{z}}
\newcommand{\bff}{\mathbf{f}}
\newcommand{\bu}{\mathbf{u}}
\newcommand{\bv}{\mathbf{v}}

\newcommand{\bh}{\mathbf{h}}
\newcommand{\bg}{\mathbf{g}}

\newcommand{\bI}{\mathbf{I}}
\newcommand{\bG}{\mathbf{G}}
\newcommand{\bH}{\mathbf{H}}
\newcommand{\bzero}{\mathbf{0}}

\newcommand{\bK}{\mathbf{K}}
\newcommand{\bC}{\mathbf{C}}

\newcommand{\balpha}{\bm{\alpha}}

\newcommand{\bpsi}{\bm{\psi}}

\newcommand{\btheta}{\bm{\theta}}

\newcommand{\beps}{\bm{\epsilon}}

\DeclareMathOperator*{\argmax}{\arg\!\max}

\newtheorem{asp}{Assumption}
\newtheorem{thm}{Theorem}
\newtheorem{lem}{Lemma}
\crefname{lem}{lemma}{lemmas}
\newtheorem{prop}{Proposition}

\usepackage[english]{babel}
\usepackage[accepted]{icml2018}

\icmltitlerunning{A Spectral Approach to Gradient Estimation for Implicit Distributions}

\begin{document}

\twocolumn[
\icmltitle{A Spectral Approach to Gradient Estimation for Implicit Distributions}



\icmlsetsymbol{equal}{*}

\begin{icmlauthorlist}
\icmlauthor{Jiaxin Shi}{tsinghua}
\icmlauthor{Shengyang Sun}{toronto}
\icmlauthor{Jun Zhu}{tsinghua}
\end{icmlauthorlist}

\icmlaffiliation{tsinghua}{Dept. of Comp. Sci. \& Tech., BNRist Center, State Key Lab for Intell. Tech. \& Sys., THBI Lab, Tsinghua University}
\icmlaffiliation{toronto}{Dept. of Comp. Sci., University of Toronto}

\icmlcorrespondingauthor{Jiaxin Shi}{shijx15@mails.tsinghua.edu.cn}
\icmlcorrespondingauthor{Jun Zhu}{dcszj@tsinghua.edu.cn}

\icmlkeywords{Machine Learning, ICML}

\vskip 0.3in
]



\printAffiliationsAndNotice{}  

\begin{abstract}
Recently there have been increasing interests in learning and inference with implicit distributions (i.e., distributions without tractable densities). To this end, we develop a gradient estimator for implicit distributions based on Stein's identity and a spectral decomposition of kernel operators, where the eigenfunctions are approximated by the Nystr{\"o}m method. Unlike the previous works that only provide estimates at the sample points, our approach directly estimates the gradient function, thus allows for a simple and principled out-of-sample extension. We provide theoretical results on the error bound of the estimator and discuss the bias-variance tradeoff in practice. The effectiveness of our method is demonstrated by applications to gradient-free Hamiltonian Monte Carlo and variational inference with implicit distributions. Finally, we discuss the intuition behind the estimator by drawing connections between the Nystr{\"o}m method and kernel PCA, which indicates that the estimator can automatically adapt to the geometry of the underlying distribution.
\end{abstract}
\vspace{-0.4cm}
\section{Introduction}
\label{sec:intro}

Recently there have been increasing interests in learning and inference with implicit distributions, i.e., distributions defined by a sampling process but without tractable densities. Popular examples include Generative adversarial networks (GAN)~\citep{goodfellow2014generative,mohamed2016learning}. Compared to the explicit likelihoods (e.g., Gaussian) in other deep generative models such as variational autoencoders (VAE)~\citep{kingma2013auto}, implicit distributions are shown able to capture the complex data manifold that lies in a high dimensional space, leading to more realistic samples generated by GAN than other models. Besides, as the constraint of requiring an explicit density is removed, implicit distributions are treated as more flexible variants of variational families for approximate inference~\citep{ranganath2016operator,liu2016two,mescheder2017adversarial,tran2017hierarchical,huszar2017variational,li2018gradient,shi2018kernel}. 

Despite that it is appealing to use flexible implicit distributions, which capture complex correlations and manifold structures, deploying them in practical scenes is still challenging. This is because most learning and inference algorithms require optimizing some divergences between two distributions, which often rely on evaluating the densities of them. However, the density of an implicit distribution is intractable and we only have access to its samples. Previous works have explored two directions to solve the problem. One is to first approximate the optimization objective with the samples and then use the approximation to guide the learning procedure. Many works in this direction are based on the fact that the density ratio between two distributions can be estimated from their samples, by a probabilistic classifier (also known as the discriminator) trained in an adversarial game~\citep{donahue2016adversarial,dumoulin2016adversarially,mescheder2017adversarial,tran2017hierarchical,huszar2017variational}, or by kernel-based estimators~\citep{shi2018kernel}. The other direction is to estimate the gradients instead of the objective. \citet{li2018gradient} propose the Stein gradient estimator for the log density of an implicit distribution. It is based on a ridge regression that inverts a generalized version of Stein's identity~\citep{gorham2015measuring,liu2016kernelized}. As argued in their work, this approach is more direct and avoids probable arbitrarily diverse gradients provided by the approximate objective. In this paper, we focus on the latter direction.

Though the Stein gradient estimator has been shown to be a fast and easy way to obtain gradient estimates for implicit models. It is still limited by its simple formulation, i.e., the unjustified choice of both the test function (the kernel feature mapping) and the regularization scheme (the Frobenius-norm regularization used in ridge regression). The problem has deeper implications. For instance, no theoretical results
have been established for the estimator. Moreover, there is no principled way to obtain gradient estimates at positions out of the sample points.

In this paper, we develop a novel gradient estimator for implicit distributions, which is called the \emph{Spectral Stein Gradient Estimator} (SSGE). To approximate the gradient function of the log density (i.e., $\nabla_{\bx}\log q(\bx)$), SSGE expands it in terms of the eigenfunctions of a kernel-based operator. These eigenfunctions are orthogonal with respect to the underlying distribution. By setting the test functions in the Stein's identity to these eigenfunctions, we can take advantage of their orthogonality to obtain a simple solution. The eigenfunctions in the solution are then approximated by the Nystr{\"o}m method~\citep{nystrom1930praktische,baker1997numerical,williams2001using}. Unlike the Stein gradient estimator~\citep{li2018gradient}, our approach allows for a direct and principled out-of-sample extension.
 Moreover, we provide theoretical analysis on the error bound of SSGE and discuss the bias-variance tradeoff in practice. We also discuss its effectiveness in reducing the curse of dimensionality by drawing connections to kernel PCA.

\vspace{-0.15cm}
\section{Background}
\label{sec:bg}

In this section we briefly introduce the Nystr{\"o}m method and the Stein gradient estimator.

\vspace{-0.15cm}
\subsection{The Nystr{\"o}m Method} \label{sec:nystrom}

The Nystr{\"o}m method originates as a method for approximating the solution of Fredholm integral equations of the second kind~\citep{nystrom1930praktische,baker1997numerical}. It was used by \citet{williams2001using} for estimating extensions of eigenvectors in Gaussian process regression. Specifically, the following equation for finding the eigenfunctions $\{\psi_j\}_{j\geq 1}, \psi_j \in L^2(\mathcal{X}, q)$\footnote{$L^2(\mathcal{X}, q)$ denotes the space of all square-integrable functions w.r.t. $q$.} of the covariance kernel $k(\bx, \by)$ w.r.t. the probability measure $q$ is considered:
\begin{equation} \label{eq:fredholm}
\int k(\bx, \by)\psi(\by)q(\by) d\by = \mu \psi(\bx).
\end{equation}
And there is a constraint that the eigenfunctions $\{\psi_j\}_{j\geq 1}$ are orthonormal under $q$:
\begin{equation} \label{eq:eigenfun-orthon}
\int \psi_i(\bx)\psi_j(\bx)q(\bx) d\bx = \delta_{ij},
\end{equation}
where $\delta_{ij} = \mathds{1}[i = j]$.
Approximating the left side of \cref{eq:fredholm} with its unbiased Monte Carlo estimate using i.i.d. samples $\{\bx^1, \dots, \bx^M\}$ from $q$ and applying the equation to these samples, we obtain
\begin{equation} \label{eq:fredholm-eigen}
\frac{1}{M}\bK\bpsi \approx \mu\bpsi,
\end{equation}
where $\bK$ is the Gram matrix: $\bK_{ij} = k(\bx^i, \bx^j)$, and $\bpsi = \left[\psi(\bx^1), \dots, \psi(\bx^M)\right]^\top$. This is an eigenvalue problem for $\bK$. We compute the eigenvectors $\bu_1, \dots, \bu_J$ with the $J$ largest eigenvalues $\lambda_1 \geq \dots \geq \lambda_J$ for $\bK$. Now we have the solutions of \cref{eq:fredholm-eigen} by comparing against $\bK\bu_j=\lambda_j\bu_j$:
\begin{align}
\psi_j(\bx^m) &\approx \sqrt{M}u_{jm},\quad m=1, \dots, M, \label{eq:eigenfun-vec}\\
\mu_j &\approx \frac{\lambda_j}{M}. \label{eq:eigenvalue-fun-vec}
\end{align}
Note that the scaling factor in \cref{eq:eigenfun-vec} is due to the empirical constraint translated from \cref{eq:eigenfun-orthon}: $\frac{1}{M}\sum_{m=1}^M\psi_i(\bx^m)\psi_j(\bx^m) \approx \delta_{ij}$. \citet{baker1997numerical} shows that for a fixed kernel $k$, $\frac{\lambda_j}{M}$ converges to $\mu_j$ in the limit as $M\to\infty$.

Plugging these solutions back into \cref{eq:fredholm}, we get the Nystr{\"o}m formula for approximating the value of the $j$th eigenfunction at any point $\bx$:
\begin{equation} \label{eq:nystrom}
\psi_j(\bx) \approx \hat{\psi}_j(\bx) = \frac{\sqrt{M}}{\lambda_j}\sum_{m=1}^M u_{jm}k(\bx, \bx^m).
\end{equation}
The Nystr{\"o}m method has been shown to be a thread linking many dimension reduction methods such as kernel PCA~\citep{scholkopf1998nonlinear}, multidimensional scaling (MDS)~\citep{borg2005modern}, local linear embedding (LLE)~\citep{roweis2000nonlinear}, Laplacian eigenmaps~\citep{belkin2003laplacian}, and spectral clustering~\citep{weiss1999segmentation}, unifying their out-of-sample extensions~\citep{bengio2004learning,bengio2004out,burges2010dimension}. We will discuss these connections later.

\vspace{-0.15cm}
\subsection{Stein's Identity and Stein Gradient Estimator}
\label{sec:stein}

Recent developments in Stein discrepancy and its kernelized extensions~\citep{gorham2015measuring,chwialkowski2016kernel,liu2016kernelized,liu2016stein} have renewed the interests in Stein's method, which is a classic tool in statistics. Central to these works is an equation that generalizes the original Stein's identity~\citep{stein1981estimation}, shown in the following theorem.

\begin{thm}[\citealt{gorham2015measuring,liu2016kernelized}] \label{thm:stein-identity}
	Assume that $q(\bx)$ is a continuous differentiable probability density supported on $\mathcal{X}\subset\mathbb{R}^d$. $\bh: \mathcal{X} \to \mathbb{R}^{d'}$ is a smooth vector-valued function $\bh(\bx) = \left[h_1(\bx), h_2(\bx),\dots, h_{d'}(\bx)\right]^\top$, and $\forall i\in 1,\dots,d'$, $h_i$ is in the \textbf{Stein class} of $q$, i.e.,
    \begin{equation}\label{eq:stein-cond}
    \int_{\bx \in \mathcal{X}} \nabla_{\bx}\left(h_i(\bx)q(\bx)\right)d\bx = 0.
    \end{equation}	
	Then the following identity holds:
	\begin{equation}\label{eq:stein}
	\mathbb{E}_q[\bh(\bx)\nabla_{\bx}\log q(\bx)^\top + \nabla_{\bx}\bh(\bx)] = \bzero.
	\end{equation}
\end{thm}

The condition (\ref{eq:stein-cond}) can be easily checked using integration by parts or divergence theorem. Specifically, when $\mathcal{X} = \mathbb{R}^d$, \cref{eq:stein-cond} holds if $\lim_{\|\bx\|\to\infty}\bh(\bx)q(\bx) = \bzero$; when $\mathcal{X}$ is a compact subset of $\mathbb{R}^d$ with piecewise smooth boundary $\partial\mathcal{X}$, \cref{eq:stein-cond} holds if $\bh(\bx)q(\bx) = \bzero, \forall \bx\in \partial{\mathcal{X}}$. Here $\bh(\bx)$ is called the test function. We can check that for a RBF kernel $k$, and for any fixed $\bx$, $k(\bx, \cdot)$ and $k(\cdot, \bx)$ are in the Stein class of continuous differentiable densities supported on $\mathbb{R}^d$.

Because the expectation in \cref{eq:stein} can be approximated by Monte Carlo estimates, the identity connects $\nabla_{\bx} \log q(\bx)$ and the samples from $q$.
Inspired by this, \citet{li2018gradient} propose the Stein gradient estimator, which inverts \cref{eq:stein} to obtain estimates of $\nabla_{\bx}\log q(\bx)$ at the sample points. Below we briefly review their method. Specifically, consider $M$ i.i.d. samples $\bx^{1:M}$ drawn from $q(\bx)$. We define two matrices $\bH=\left[\bh(\bx^1), \cdots, \bh(\bx^M)\right] \in \mathbb{R}^{d'\times M}$ and $\bG=\left[\nabla_{\bx^1}\log q(\bx^1), \cdots, \nabla_{\bx^M}\log q(\bx^M)\right]^\top \in \mathbb{R}^{M \times d}$. Monte Carlo sampling with \cref{eq:stein} shows that
\begin{equation} \label{eq:stein-grad-mc}
    -\frac{1}{M}\bH\bG \approx \overline{\nabla_{\bx}\bh},
\end{equation}
where $\overline{\nabla_{\bx}\bh} = \frac{1}{M}\sum_{m=1}^M \nabla_{\bx^m}\bh(\bx^m) \in \mathbb{R}^{d'\times d}$, $\nabla_{\bx^m}\bh(\bx^m) = \left[\nabla_{\bx^m}h_1(\bx^m), \dots, \nabla_{\bx^m}h_{d'}(\bx^m)\right]^\top$. \Cref{eq:stein-grad-mc} inspires the following ridge regression problem:
\begin{equation*} 
    \underset{\hat{\bG} \in \mathbb{R}^{M \times d}}{\mathrm{argmin}} \|\overline{\nabla_{\bx}\bh} + \frac{1}{M}\bH \hat{\bG}\|^2_F + \frac{\eta}{M^2}\|\hat{\bG}\|^2_F,
\end{equation*}
where $\|\cdot\|^2_{F}$ denotes the Frobenius norm of a matrix and $\eta > 0$ is the regularization coefficient. It has an analytic solution that
\begin{equation}
\label{eq:stein-grad-est}
    \hat{\bG}^{Stein} = - M(\bK + \eta \bI)^{-1}\bH^\top \overline{\nabla_{\bx}\bh},
\end{equation}
where $\bK = \bH^\top\bH$. By noticing that $\bK_{ij} = \bh(\bx^i)^\top\bh(\bx^j)$ and applying the kernel trick, we have $\bK_{ij} = k(\bx^i, \bx^j)$, where $k:\!\mathbb{R}^d\times\mathbb{R}^d\to\mathbb{R}$ is a positive definite kernel. 
Similarly we can show $(\bH^\top \overline{\nabla_{\bx}\bh})_{ij} = \frac{1}{M}\sum_{m=1}^M \nabla_{x_j^m}k(\bx^i, \bx^m)$. With the kernel trick, the above derivation implicitly set the test function to be the feature mapping $\bh: \mathbb{R}^d\to \mathcal{H}, \bh(\bx) = k(\bx, \cdot)$, where $\mathcal{H}$ is the Reproducing Kernel Hilbert Space (RKHS) induced by $k$. 

Though introducing the kernel trick enhances the expressiveness of the Stein gradient estimator, the choice of the test function is not well justified. It is unclear whether \cref{eq:stein} holds when the test function is set to the kernel feature mapping $\bh(\bx) = k(\bx, \cdot)$, which maps from $\mathbb{R}^d$ to $\mathcal{H}$ instead of the common $\mathbb{R}^{d'}$. \citet{li2018gradient} show that their approach is equivalent to minimizing a regularized version of the V-statistics of kernel Stein discrepancy~\citep{liu2016kernelized}, which has proved to be effective in testing goodness-of-fit. However, it is still unclear whether the test power is sufficient due to the added Frobenius-norm regularization term. 
Besides, \cref{eq:stein-grad-est} only gives the gradient estimates at the sample points. For out-of-sample prediction at a test point, two choices are proposed in \citet{li2018gradient}: one is by adding the test point to the sample points and re-compute \cref{eq:stein-grad-est}, which could be computationally demanding. We will refer to this out-of-sample extension 
as \emph{Stein$^+$}. The other choice is to refit a parametric estimator (a linear combination of RBF kernels), which is cheaper but less accurate. Both approaches assume that the test points are also sampled from $q$. However, they are unjustified when the assumption is not satisfied. Since the latter is an approximation to the former, we will only compare with Stein$^+$ in the experiments for out-of-sample predictions.

\vspace{-0.15cm}
\section{Method}

In this section we derive a gradient estimator for implicit distributions called the \emph{Spectral Stein Gradient Estimator} (SSGE).
Unlike the previous Stein gradient estimator that only provides estimates at the sample points, SSGE directly estimates the gradient function and thus allows simple and principled out-of-sample predictions. We also provide theoretical analysis on the error bound of SSGE.

\vspace{-0.15cm}
\subsection{Spectral Stein Gradient Estimator}
\label{sec:derivation}

To begin with, let $\bx$ denote a $d$-dimensional vector in $\mathbb{R}^d$. Consider an implicit distribution $q(\bx)$ supported on $\mathcal{X}\subset \mathbb{R}^d$, from which we observe $M$ i.i.d. samples $\bx^{1:M}$. We denote the target gradient function to estimate by $\bg:\mathcal{X}\to \mathbb{R}^d$: $
	\bg(\bx) = \nabla_{\bx}\log q(\bx).$
The $i$th component of the gradient is $g_i(\bx) = \nabla_{x_i}\log q(\bx)$. We assume $g_1,\dots,g_d \in L^2(\mathcal{X}, q)$.
As introduced in \cref{sec:nystrom}, $\{\psi_j\}_{j\geq 1}$ form an orthonormal basis of $L^2(\mathcal{X}, q)$. So we can expand $g_i(\bx)$ into the following spectral series:
\begin{equation} \label{eq:grad-series}
g_i(\bx) = \sum_{j=1}^\infty \beta_{ij}\psi_j(\bx).
\end{equation}
Below we will show how to estimate the coefficients $\beta_{ij}$.
According to \Cref{thm:stein-identity}, we have the following proposition.

\begin{prop} \label{prop:eigen-boundary}
	If $k(\cdot, \cdot)$ has continuous second order partial derivatives, and both $k(\bx, \cdot)$ and $k(\cdot, \bx)$ are in the Stein class of $q$, the following set of equations hold true:
	\begin{equation} \label{eq:eigen-stein}
		\mathbb{E}_q[\psi_j(\bx)g(\bx) + \nabla_{\bx}\psi_j(\bx)] = \bzero,\quad j=1,2 \dots, \infty.
	\end{equation}
\end{prop}
\begin{proof}
	We only need to prove that $\psi_j(\bx)$ is in the Stein class of $q$. See \Cref{app:proof-eigen-stein} for details.
\end{proof}

Substituting \cref{eq:grad-series} into \cref{eq:eigen-stein} and using the orthonormality of $\{\psi_j\}_{j\geq 1}$, we can show that
\begin{equation*}
	\beta_{ij} = -\mathbb{E}_q \nabla_{x_i}\psi_j(\bx).
\end{equation*}
To estimate $\beta_{ij}$, we need an approximation of $\nabla_{x_i}\psi_j(\bx)$. The key observation is that derivatives can be taken w.r.t. both sides of \cref{eq:fredholm}:
\begin{equation} \label{eq:grad-fredholm}
	\begin{aligned}
	\mu_j\nabla_{x_i}\psi_j(\bx) &= \nabla_{x_i}\int k(\bx, \by)\psi_j(\by)q(\by) d\by \\
	&= \int \nabla_{x_i}k(\bx, \by)\psi_j(\by)q(\by) d\by.
	\end{aligned}
\end{equation}
Monte-Carlo sampling with \cref{eq:grad-fredholm}, we have an estimate of $\nabla_{x_i}\psi_j(\bx)$:
\begin{equation} \label{eq:grad-eigen-est}
	\hat{\nabla}_{x_i}\psi_j(\bx) \approx \frac{1}{\mu_jM}\sum_{m=1}^M \nabla_{x_i} k(\bx, \bx^m)\psi_j(\bx^m).
\end{equation}
Substituting \cref{eq:eigenfun-vec,eq:eigenvalue-fun-vec} into \cref{eq:grad-eigen-est} and comparing with \cref{eq:nystrom}, we can show
\begin{equation} \label{eq:grad-nystrom-is-good}
	\hat{\nabla}_{x_i}\psi_j(\bx) \approx \nabla_{x_i}\hat{\psi}_j(\bx).
\end{equation}
Perhaps surprisingly, \Cref{eq:grad-nystrom-is-good} indicates that $\nabla_{x_i}\hat{\psi}_j(\bx)$ is a good approximation to $\nabla_{x_i}\psi_j(\bx)$\footnote{This  does not hold for general functions.}. In fact, as we shall see in \Cref{thm:consistency},
the error introduced by Nystr{\"o}m approximation is negligible with high probability as $M\to \infty$.

Now truncating the series expansion to the first $J$ terms and plugging in the Nystr{\"o}m approximations of $\{\psi_j\}_{j=1}^J$, we get our estimator:
\begin{align}
	&\hat{g}_i(\bx) = \sum_{j=1}^{J}\hat{\beta}_{ij}\hat{\psi}_j(\bx),\label{eq:grad-est}\\
	&\hat{\beta}_{ij} = -\frac{1}{M}\sum_{m=1}^M \nabla_{x_i}\hat{\psi}_j(\bx^m) \label{eq:coef-est},
\end{align}
where $\hat{\psi}_j$ is the Nystr{\"o}m approximation of $\psi_j$ as in \Cref{sec:nystrom}. We use RBF kernels in all experiments.

\textbf{Computational Cost}\; The spectral gradient estimator $\hat{g}_i(\bx)$ depends on $\hat{\beta}_{ij}$ and $\hat{\psi}_j(\bx)$. To compute them, the computational bottleneck lies in computing the Gram matrix and its eigendecomposition, which have complexity $O(M^2d)$ and $O(M^3)$, respectively. Therefore the computational cost of constructing the estimator is $O(M^3+M^2d)$.
For prediction, given $\bx \in \mathbb{R}^d$, evaluating $\hat{g}_i(\bx)$ has cost $O(M(d+J))$.
In comparison, the Stein gradient estimator directly approximates gradients at the sample points, which involves computing the Gram matrix and its inverse. Thus the overall complexity is also $O(M^3+M^2d)$. Note that SSGE only requires the $J$ largest eigenvalues and corresponding eigenvectors, efficient algorithms \citep{parlett1979lanczos} might be applied to further reduce its complexity.

\vspace{-0.15cm}
\subsection{Theoretical Results}

Following the derivation in \Cref{sec:derivation}, we analyze the theoretical properties of the resulting estimator in \cref{eq:grad-est,eq:coef-est}. To be clear, we formally restate the assumptions that have been made in the derivation.

\begin{asp} \label{asp:kernel-boundary}
	$k(\bx, \cdot)$ and $k(\cdot, \bx)$ are in the Stein class of q.
\end{asp}
\begin{asp} \label{asp:square-integrable}
	$g_i(\bx) \in L^2(\mathcal{X}, q)$, $i=1,\dots,d$, i.e., $$
	\int g_i(\bx)^2 q(\bx)\;d\bx = \sum_{j=1}^\infty \beta_{ij}^2 \le C < \infty.
	$$
\end{asp}
\begin{asp} \label{asp:unique-eigen}
    $\mu_1 > \mu_2 > \dots > \mu_J > 0.$
\end{asp}
Note that \Cref{asp:kernel-boundary} holds for RBF kernels. \Cref{asp:square-integrable} is necessary for $g_i(\bx)$'s being possible to be expanded into the spectral series. We need \Cref{asp:unique-eigen} since our derivation is based on several well-studied bounds of Nystr{\"o}m approximation, i.e., \Cref{lem:nystrom-vs-eigen,lem:nystrom-inner} in \Cref{app:theory}~\citep{sinha2009semi,izbicki2014high}. Note that when this assumption does not hold, we could proceed as in~\citet{rosasco2010learning} (Theorem 12) and derive the error bound in a similar way.

\begin{thm}[Error Bound of SSGE, proof in \Cref{app:theory}] \label{thm:consistency}
	Given the above assumptions,
	the error $\int |\hat{g}_i(\bx) - g_i(\bx)|^2q(\bx)\;d\bx$ is bounded by
	
	\begin{equation} \label{eq:error-bound}
		\begin{aligned}
		J^2 &\left(O_p\left(\frac{1}{M}\right) +  O_p\left(\frac{C}{\mu_J\Delta_J^2M}\right)\right) + \\ &JO_p\left(\frac{C}{\mu_J\Delta_J^2M}\right) + \|g_i\|^2_{\mathcal{H}}O(\mu_J),
		\end{aligned}
	\end{equation}
	where $\Delta_J = \min_{1\leq j\leq J}\left|\mu_j - \mu_{j+1}\right|$, $O_p$ is the Big O notation in probability.
\end{thm}
The first three terms in \cref{eq:error-bound} are the sample errors caused by the Nystr{\"o}m approximation, which we call the \emph{estimation error}. It is negligible with high probability as $M\to\infty$. The last term is caused by the bias introduced by the truncation, which we call the \emph{approximation error}. From the bound we can observe a tradeoff between the estimation error and the approximation error. As an illustration, one may set $J$ to be as large as possible to reduce the magnitude of $\mu_J$ and thus reduce the approximation error, but it will increase the estimation error at a rate of $O_p\left(\frac{J^2}{\mu_J}\right)$.

For RBF kernels and their corresponding RKHS, smoother target functions tend to have smaller $\|g_i\|_{\mathcal{H}}^2$, and thus a tighter bound. In general, this indicates that choosing the appropriate kernel which is suitable to the target gradient function can improve the performance of the gradient estimator (by leading to a smaller $\|g_i\|_{\mathcal{H}}^2$).

\textbf{Hyperparameter Selection}\; When RBF kernels are used, SSGE has two free parameters: The kernel bandwidth $\sigma$, and the number of eigenfunctions ($J$) used in the estimate. For $\sigma$, we use the median heuristic, i.e., we set it to be the median of pairwise distances between all samples, which turns out to work well in all experiments. Below we discuss the criterion for selecting $J$, which is usually harder.

As the performance of the gradient estimator directly influences the task where it is used. The optimal choice for tuning $J$ is to apply cross-validation on the specific task. However, since  $J$ is a discrete parameter, one has to manually set a continuous interval and bin it so that the commonly used black-box hyperparameter-search methods (e.g., Bayesian optimization) are applicable. However, this approach does not take the magnitude of eigenvalues into consideration. Observing this, we propose that, instead of directly tuning $J$, we could tune a threshold $\bar{r}$ for the percentage of remaining eigenvalues:
\begin{equation*}
J = \argmax_{J'}\;r_{J'},\;
\mathrm{\textit{s.t.}} \; r_{J'} = \frac{\sum_{j=1}^{J'}\lambda_j}{\sum_{j'=1}^M\lambda_{j'}}, r_{J'}\leq \bar{r}.
\end{equation*}
Note that searching $\bar{r}$ may still not be easy due to the non-smooth validation surface. But in experiments we found that $\bar{r}$ values in $[0.95, 0.99]$ usually work well.

\vspace{-0.15cm}
\subsection{Gradient Estimation for Entropy}
\label{sec:entropy}
Above we have derived a gradient estimator for the log density of an implicit distribution. Now we discuss a useful extension of it. Consider the situation where we need to optimize the entropy $\mathbb{H}(q)=-\mathbb{E}_q \log q$ of an implicit distribution $q_{\phi}(\bx)$ w.r.t. its parameters $\phi$. When $\bx$ is continuous and can be reparameterized~\citep{kingma2013auto}, e.g., $\bx = f(\beps;\phi), \beps\sim \mathcal{N}(\bzero, \bI)$, it can be shown (see \Cref{app:entropy-grad}) that
\begin{equation}\label{eq:entropy-grad}
	\nabla_{\phi}\mathbb{H}(q) = -\mathbb{E}_{\beps}\nabla_{\bx}\log q_{\phi}(f(\beps;\phi))\nabla_{\phi} f(\beps;\phi),
\end{equation}
where $\nabla_{\bx}\log q_{\phi}(f(\beps;\phi))$ can be easily estimated by SSGE. As we shall see in experiments, \cref{eq:entropy-grad} can be used for variational inference with implicit distributions.

\vspace{-0.15cm}
\section{Related Work} \label{sec:related}
Our work is closely related to other works on implicit distributions. Apart from the density-ratio based approaches introduced in \Cref{sec:intro}, we discuss two more directions.


\textbf{Nonparametric Inference}\; Nonparametric variational inference (VI) methods such as PMD \citep{dai2015provable} and SVGD \citep{liu2016stein} remove the need of parametric families, they keep a set of particles and gradually adjust them towards the true posterior. These particles can be viewed as samples from an implicit distribution. Instead of directly computing gradients in the sample space like us, SVGD performs functional gradient descent to transform the implicit distribution towards the true posterior. Though elegant, SVGD is limited to KL-divergence based VI problems, while our approach is generally applicable wherever gradient estimates are needed for the log density of implicit distributions.


\textbf{Kernel Exponential Families and Score Matching}\; Previous to our work, the problem of estimating gradient functions of intractable log densities has been worked on by \citet{strathmann2015gradient,pmlr-v84-sutherland18a}. They identified the problem when developing Hamiltonian Monte Carlo (HMC) under the settings where higher-order information of the target distribution is unavailable. To address it, a kernel exponential family~\citep{sriperumbudur2017density} is fit to samples along the Markov Chain trajectory by score matching~\citep{hyvarinen2005estimation}, and then serves as a surrogate of the target distribution to provide gradient estimates for HMC. We will compare to them in \Cref{sec:exp-hmc}.

\vspace{-0.15cm}
\section{Experiments}

We evaluate the proposed approach on both toy problems and real-world examples. The latter includes applications of SSGE to two widely used inference methods: Hamiltonian Monte Carlo and variational inference. Code is available at \url{https://github.com/thjashin/spectral-stein-grad}. Implementations are based on ZhuSuan~\citep{zhusuan2017}.

\vspace{-0.15cm}
\subsection{Toy Experiment}

 \begin{figure}[t]
     \vskip 0.05in
    \begin{center}
        \includegraphics[width=0.8\columnwidth]{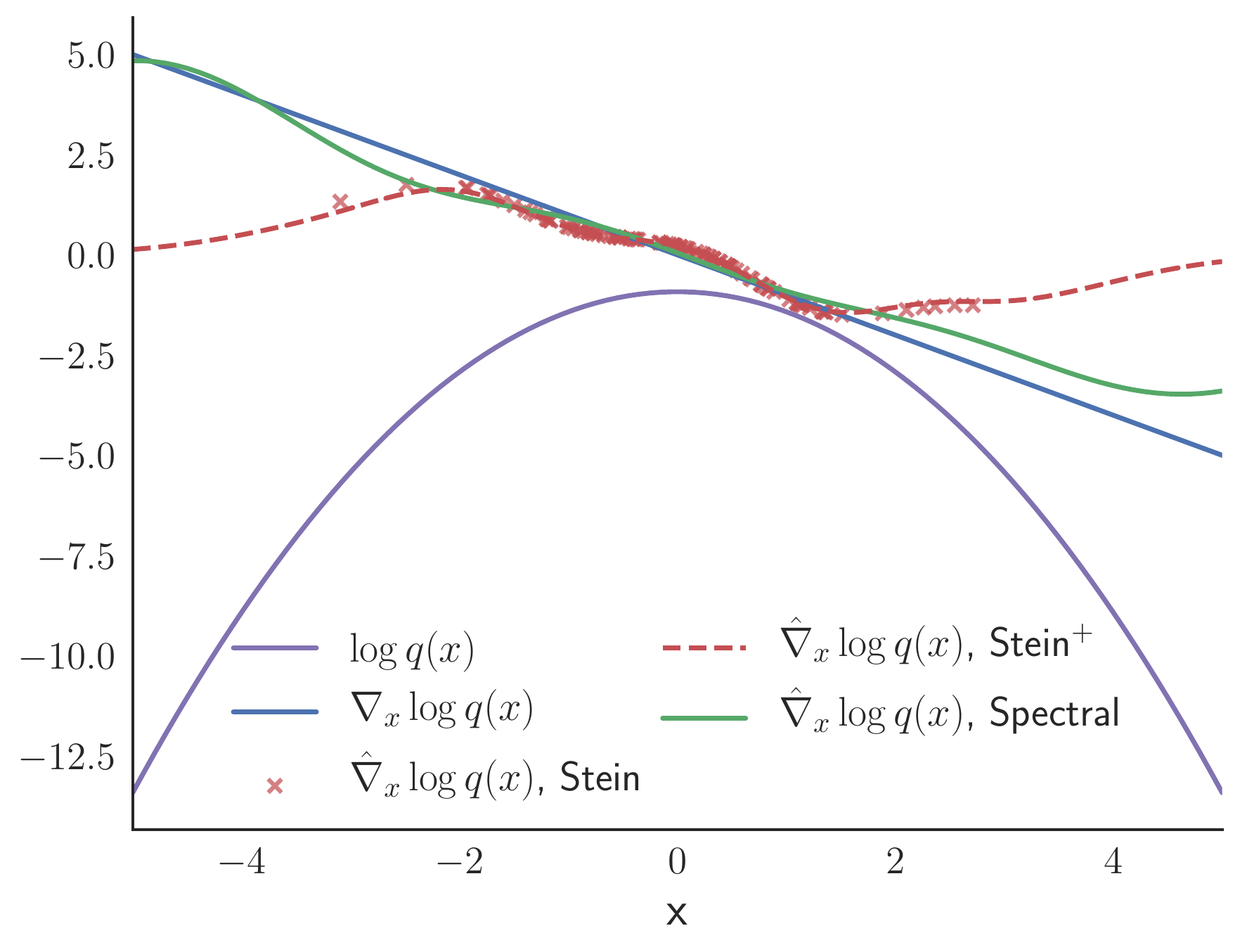}
        \vspace{-.3cm}
        \caption{Gradient estimates of the log density of $\mathcal{N}(0, 1)$.}
        \label{fig:gaussian}
    \end{center}\vspace{-0.8cm}
\end{figure}

As a simple example, we experiment with estimating the gradient function of a 1-D standard Gaussian distribution. The target log density is $\log q(x) = -\frac{1}{2}\log 2\pi -\frac{1}{2}x^2$, and the true gradient function is $\nabla_x\log q(x) = -x$. We draw $M = 100$ i.i.d. samples from $q$ for use in the estimation. In \Cref{fig:gaussian} we plot the gradients estimates produced by the Stein gradient estimator, its out-of-sample extension (Stein$^+$) (see \Cref{sec:stein}), and our approach (SSGE). Since the original Stein estimator only gives estimates at the sample points, we plot them as individual points (in red). For the regularization coefficient $\eta$ in \cref{eq:stein-grad-est}, we searched it in $\{0.001, 0.01, 0.1, 1, 10, 100\}$ and plot the best result at $\eta = 0.1$ \footnote{The criterion for selecting $\eta$ is unclear in \citet{li2018gradient}.}. For SSGE, we set $J=6$. We can see that despite all three estimators produce rather good approximation where the samples are taken densely (e.g., in $[-2, 2]$), the gradient function estimated by SSGE is notably better at the places where samples are less dense.

\begin{figure}[t]
    \begin{subfigure}[b]{0.46\columnwidth}
        \centering
        \includegraphics[width=\columnwidth]{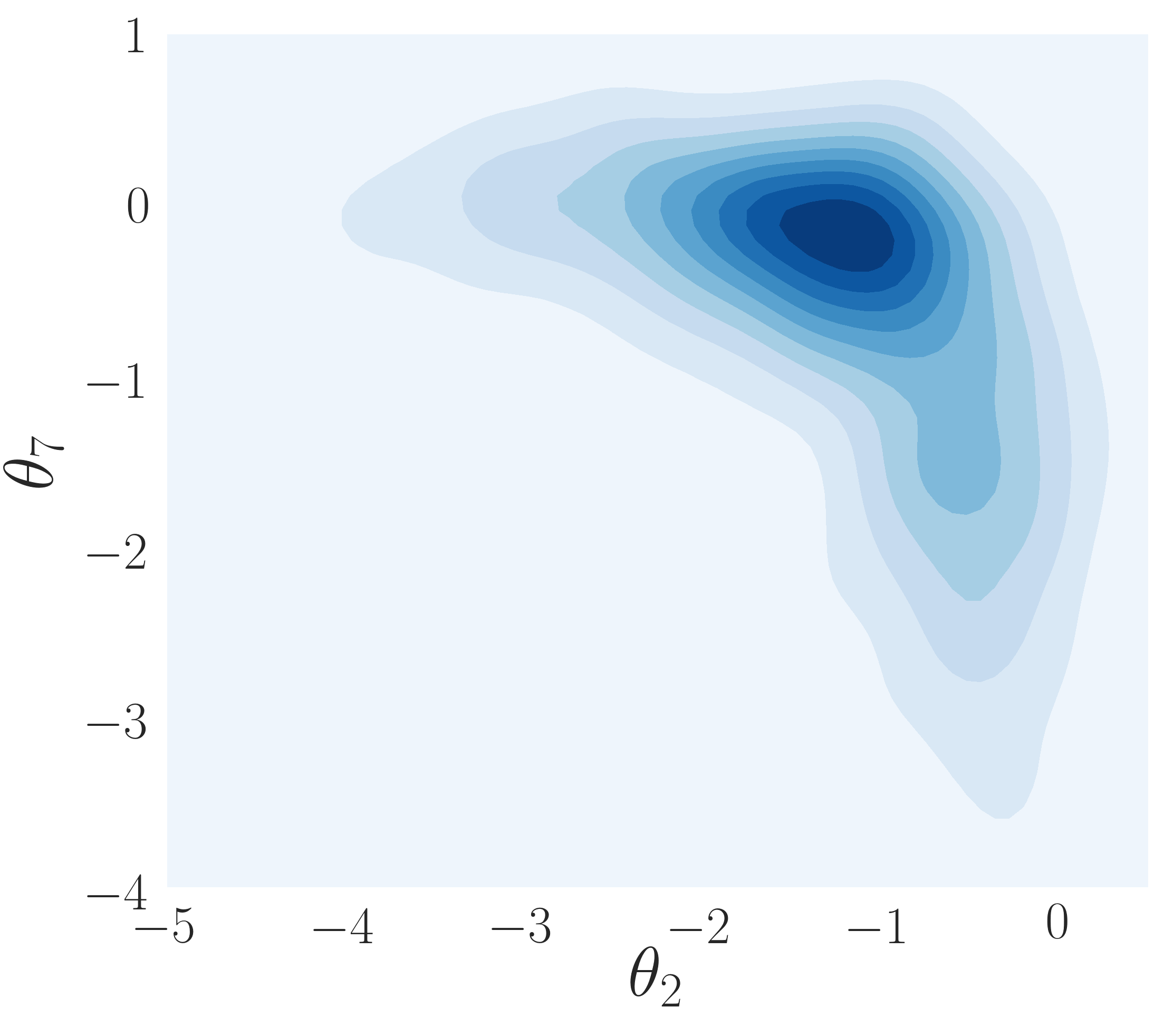}
        \caption{}
        \label{fig:gp-hyper-corr}
    \end{subfigure}
    \hfill
    \begin{subfigure}[b]{0.5\columnwidth}
        \centering
        \includegraphics[width=\columnwidth]{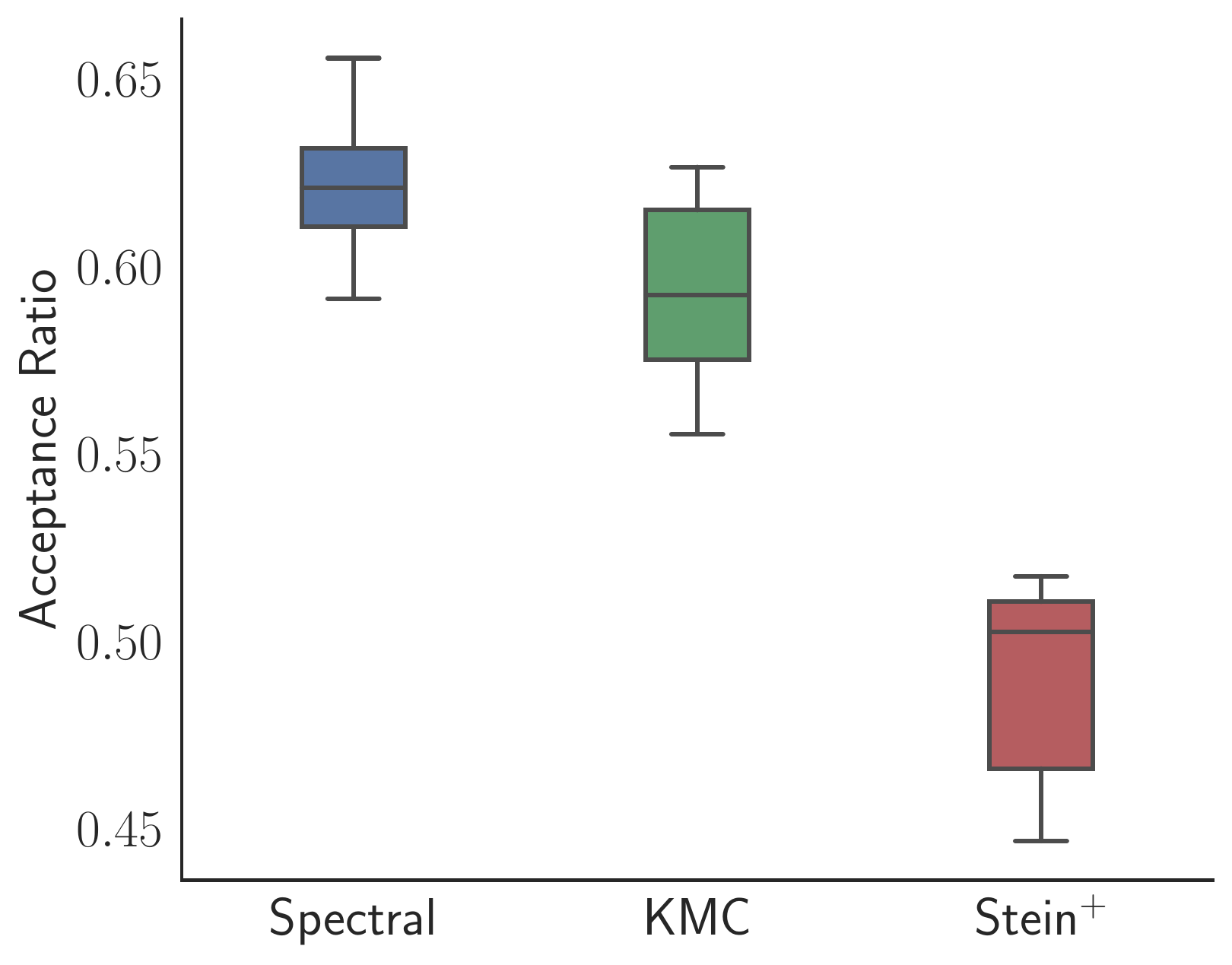}
        \vspace{-0.3cm}
        \caption{}
        \label{fig:gp-acc}
    \end{subfigure}
    \vspace{-.2cm}
    \caption{(a) Dimensions 2 and 7 of the marginal hyperparameter posterior of Gaussian Process classification on the UCI Glass dataset; (b) The average acceptance ratios of gradient-free HMC using SSGE, KMC, and Stein$^+$.}\vspace{-0.2cm}
\end{figure}

\vspace{-0.15cm}
\subsection{Gradient-free Hamiltonian Monte Carlo}
\label{sec:exp-hmc}

In this experiment we investigate the usefulness of SSGE in constructing a gradient-free HMC sampler. We follow the settings in \citet{sejdinovic2014kernel,strathmann2015gradient} and consider a Gaussian Process classification problem on the UCI Glass dataset. The goal is to infer the posterior over hyperparameters under a fully Bayesian treatment. Specifically, consider a Gaussian process whose joint distribution is over latent variables $\bff$, labels $\by$, and hyperparameters $\btheta$:
\begin{equation}
	p(\bff, \by, \btheta) = p(\btheta)p(\bff|\btheta)p(\by|\bff),
\end{equation}
where $\bff|\btheta \sim N(\bzero, \bK_{\btheta})$. $\bK_{\btheta}$ is the Gram matrix formed by the data points $\bx_{1:N} \in \mathbb{R}^D$:
\begin{equation}
	(\bK_{\btheta})_{ij} = \exp\left\{-\sum_{d=1}^D \frac{|x_{i, d} - x_{j, d}|^2}{2\ell_d^2}\right\},
\end{equation}
where we define $\theta_d = \log \ell_d^2$. The problem to consider is a binary classification between window and non-window glasses, so the likelihood is given by a logistic classifier: $p(y_i|f_i) = \frac{1}{1 + \exp(-y_if_i)}, y_i \in \{-1, 1\}$. The posterior over $\btheta$ is highly nonlinear, as shown in \Cref{fig:gp-hyper-corr}. As pointed out in previous works~\citep{murray2010slice,sejdinovic2014kernel}, sampling from the posterior of $\btheta$ is challenging, e.g., Gibbs sampling often gets stuck due to $p(\btheta|\bff, \by)$ is very sharp. One way to address this problem is the pseudo-marginal MCMC~\citep{andrieu2009pseudo}, through which a markov chain can be simulated to directly sample from $p(\btheta|\by)$. Since the likelihood $p(\by|\btheta)$ is intractable, pseudo-marginal MCMC replaces it with an unbiased Monte Carlo estimate using importance sampling:
\begin{equation}
	\hat{p}(\by|\btheta) = \frac{1}{K}\sum_{i=1}^{K}\frac{p(\by|\bff^i)p(\bff^i|\btheta)}{q(\bff^i)},\quad \bff^{1:K}\sim q(\bff).
\end{equation}
In practice $q(\bff)$ is chosen to be the Laplace approximation of $p(\bff|\by, \btheta)$. As for any pseudo-marginal MCMC scheme, the gradient information of the posterior is not available and HMC is not suitable. We have to resort to gradient-free MCMC methods. As mentioned in \Cref{sec:related}, kernel adaptive Metropolis samplers~\citep{sejdinovic2014kernel} were developed and then extended to kernel HMC (KMC)~\citep{strathmann2015gradient}. In this experiment we compare the performance of KMC and HMC with gradients estimated by SSGE and Stein$^+$.

To begin, we run 20 randomly initialized adaptive-Metropolis samplers for 30k iterations, with the first 10k samples discarded. We then keep every 400-th sample in each of the chains, and combine them to get 1k samples. These samples are treated as the ground truth.
Similar to \citet{pmlr-v84-sutherland18a}, our experiment assumes the idealized scenario where a burn-in period for collecting a sufficient number of samples has completed. This is to remove all the other factors that could have an effect on the comparison of acceptance ratios, which then only depends on the accuracy of the gradient estimation of potentials, i.e., $-\nabla_{\btheta} \log p(\btheta|\by)$. So we fit all three estimators on a random subset of $M=200$ of these samples and repeated 10 times.
For each fitted estimator, we start from a random initial point from the posterior sketch and run HMC samplers with the gradient estimates for 5k iterations. To be fair in comparing the acceptance ratios, no adaptation of HMC parameters can be used. So we randomly uses between 1 and 10 leapfrog steps of size chosen uniformly in $[0.01, 0.1]$, and a standard Gaussian momentum. The kernel bandwidths ($\sigma$) in all three estimators are determined by median heuristics (i.e., set to the median of the pairwise distances between the $M$ data points). For Stein$^+$, $\eta=0.001$. For SSGE, $\bar{r}=0.95$.


The average acceptance ratios over 10 runs are plotted in \Cref{fig:gp-acc}. We can see that SSGE clearly outperforms Stein$^+$ and is even better than the KMC algorithm, which is specially designed as a gradient-free HMC algorithm. Though KMC is carefully designed, its gradients are estimated by first fitting a kernel exponential family as a surrogate and then taking derivatives through it, while SSGE is arguably more direct.

\vspace{-0.15cm}
\subsection{Variational Inference with Implicit Distributions}

As introduced in \Cref{sec:intro}, there have been increasing interests in constructing flexible variational posteriors with implicit distributions. Specifically, for a latent-variable model $p(\bz, \bx)$ where $\bx$ and $\bz$ denote observed and latent variables, respectively, Variational Inference (VI) approximates the posterior $p(\bz|\bx)$ by maximizing the following evidence lower bound (ELBO):
\begin{equation} \label{eq:elbo}
\mathcal{L}(\bx;\phi) = \mathbb{E}_{q_{\phi}(\bz)} \log p(\bz, \bx) - \mathbb{E}_{q_{\phi}(\bz)} \log q_{\phi}(\bz),
\end{equation}
where $q_{\phi}(\bz)$ is called the variational distribution. The second term of \cref{eq:elbo} is the entropy of $q$, which is intractable for implicit distributions.  As shown in \Cref{sec:entropy}, SSGE can be used here for estimating gradients of the entropy term, thus allowing VI with implicit distributions. Below we conduct experiments on two examples: Bayesian Neural Networks (BNN) and Variational Autoencoders (VAE). Note that the original Stein gradient estimator can also be used here. In experiments, we find that despite lack of theoretical evidences, the performance of a well-tuned Stein gradient estimator is very close to SSGE when no out-of-sample predictions are required. As the emphasis of this paper is on SSGE, we only focus on verifying the accuracy of SSGE below.

\begin{figure}[t]
	\centering
    \begin{subfigure}[t]{0.46\columnwidth}
        \centering
        \includegraphics[width=\columnwidth]{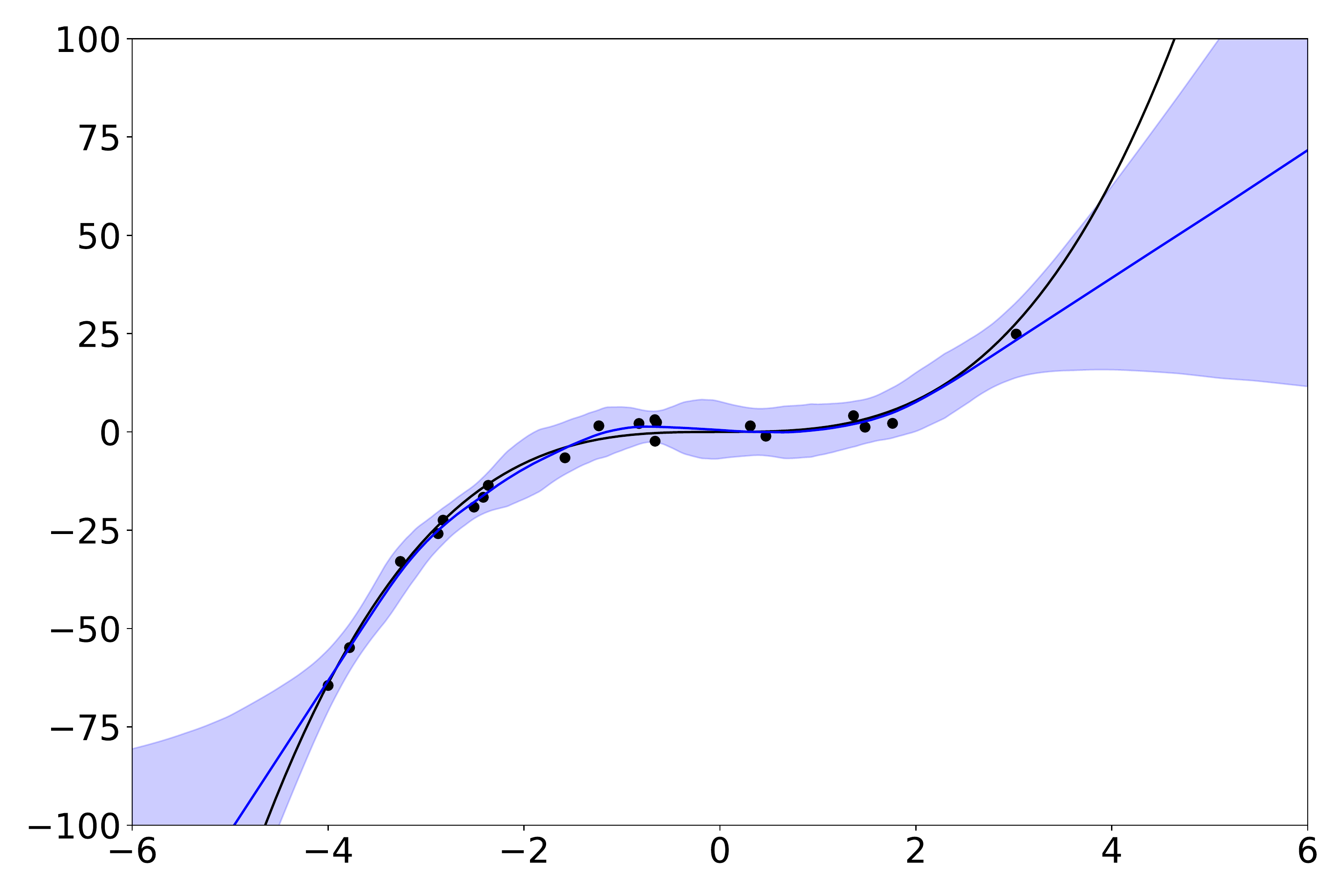}
        \caption{HMC}
        \label{fig:bnn-hmc}
    \end{subfigure}
    \begin{subfigure}[t]{0.46\columnwidth}
        \centering
        \includegraphics[width=\columnwidth]{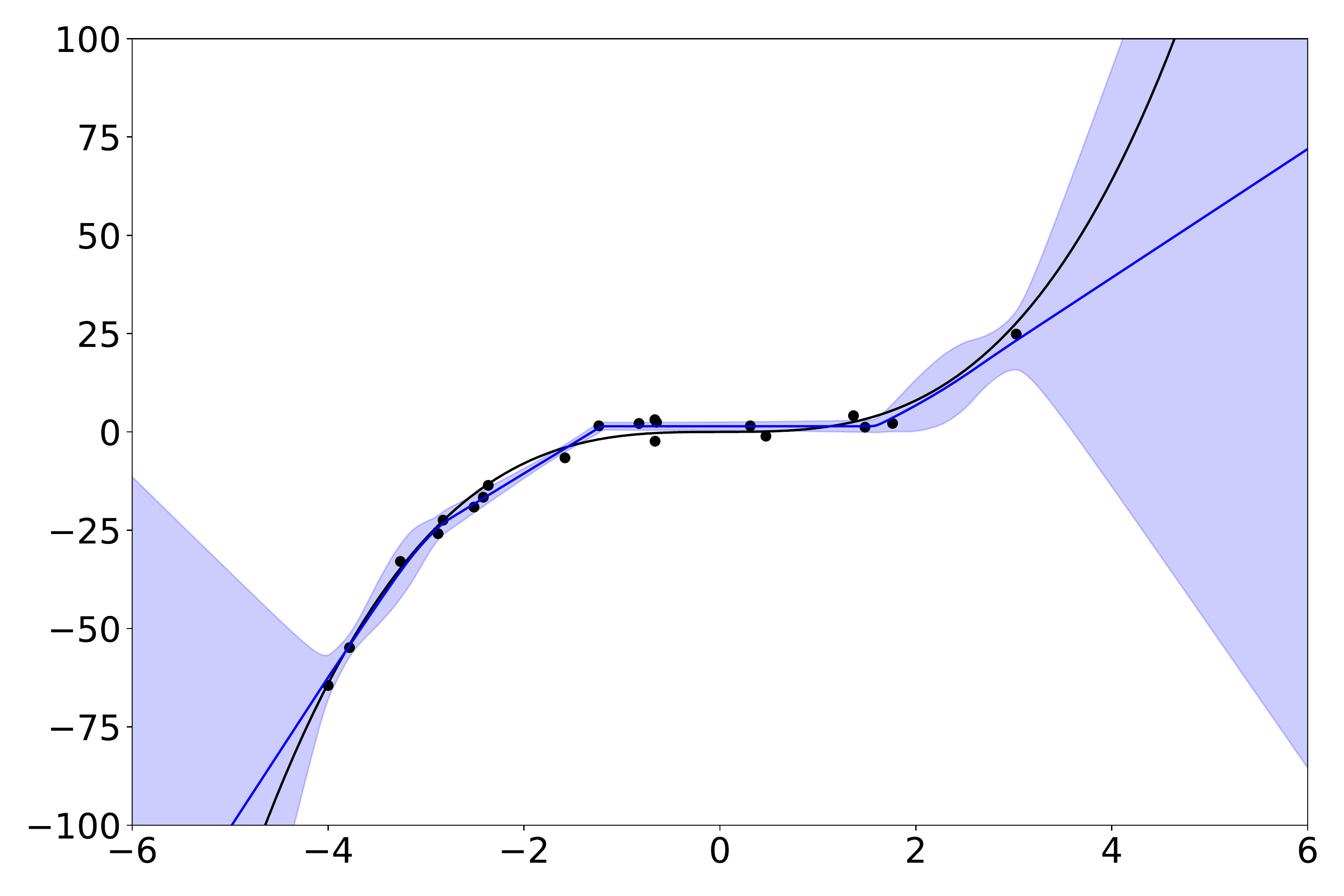}
        \caption{Spectral-Implicit}
        \label{fig:bnn-spectral}
    \end{subfigure}
    \begin{subfigure}[t]{0.46\columnwidth}
        \centering
        \includegraphics[width=\columnwidth]{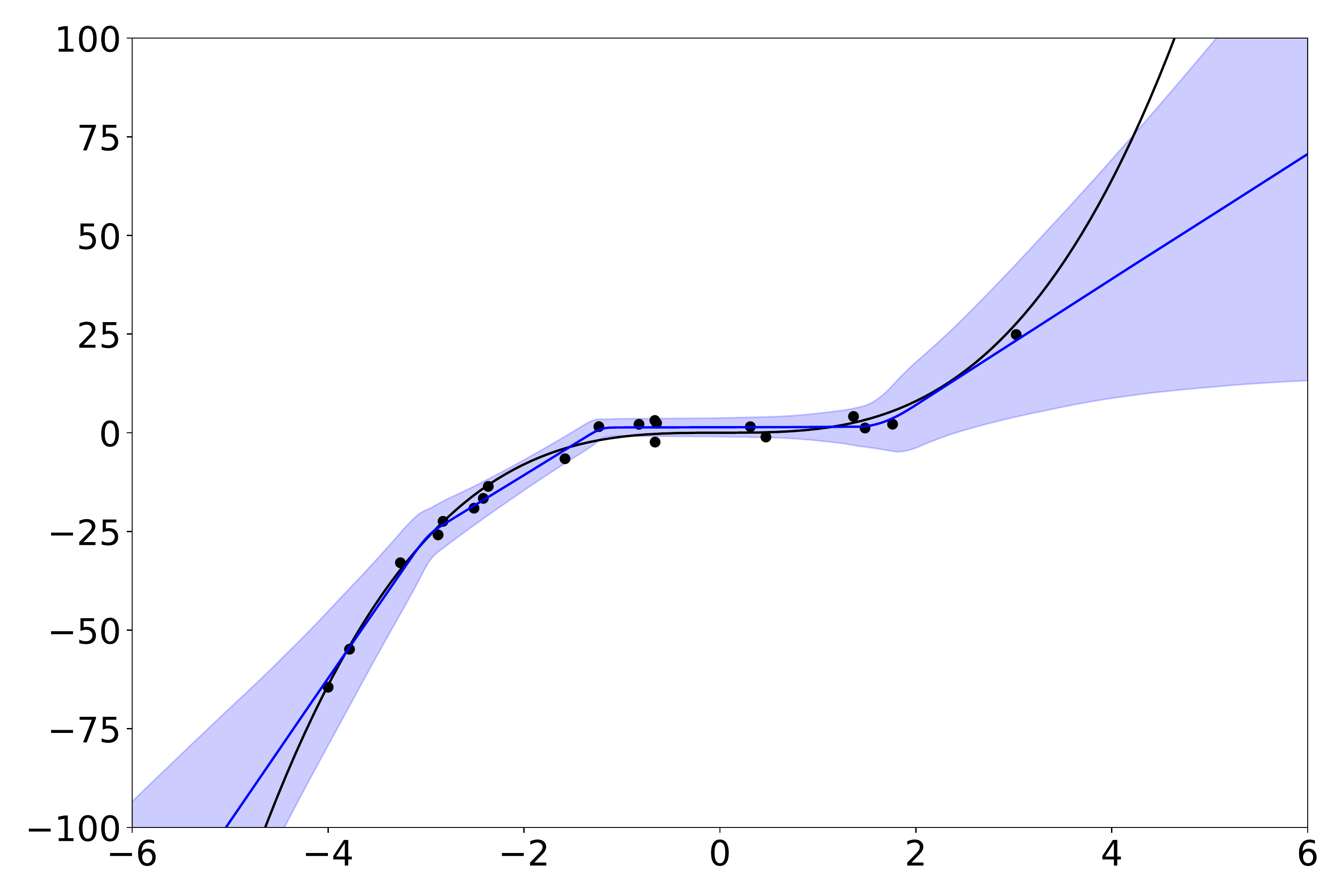}
        \caption{BBB}
        \label{fig:bnn-bbb}
    \end{subfigure}
    \begin{subfigure}[t]{0.46\columnwidth}
        \centering
        \includegraphics[width=\columnwidth]{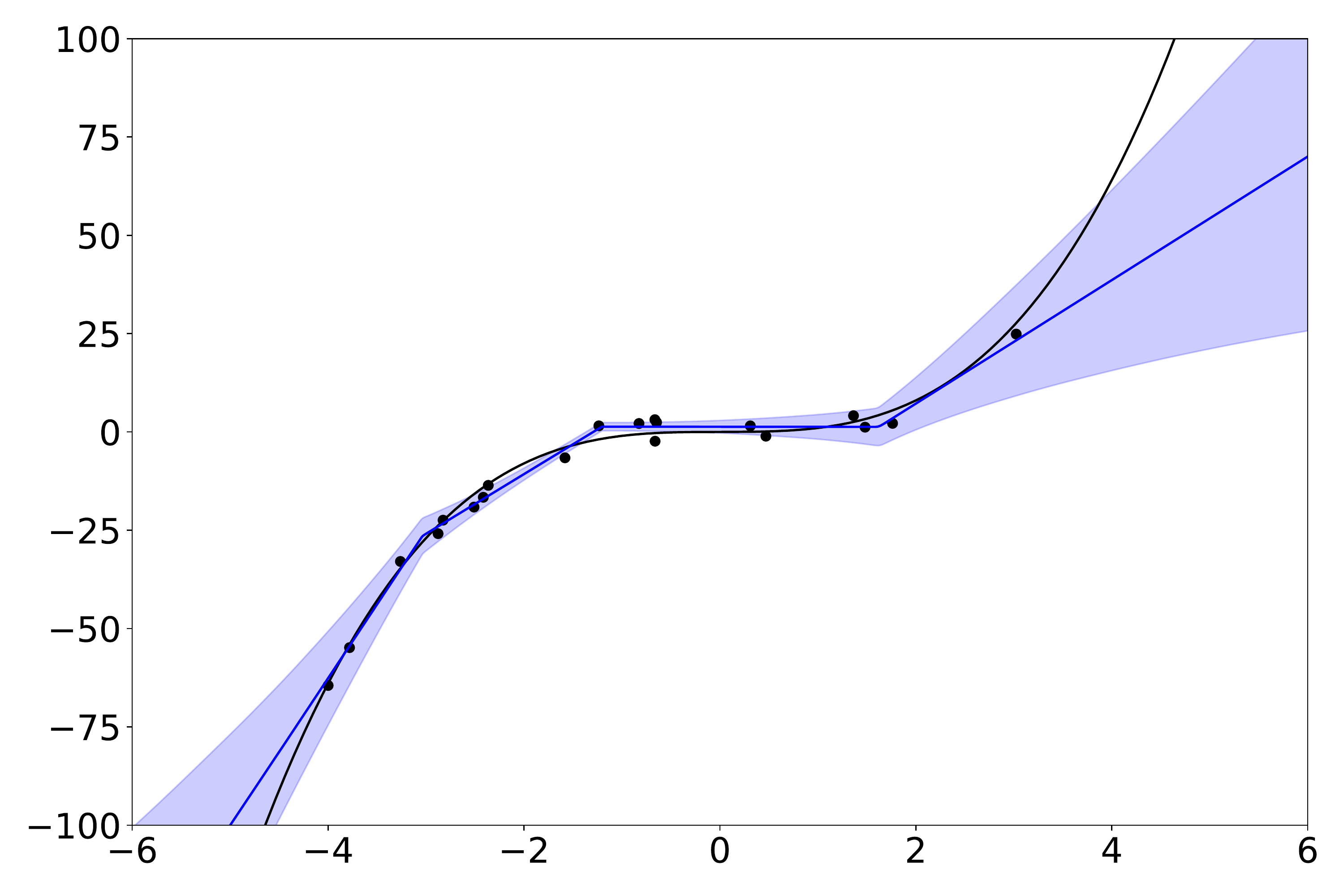}
        \caption{Spectral-Factorized}
        \label{fig:bnn-spectral-factor}
    \end{subfigure}
	\caption{Prediction results for the 1-D regression experiment. Spectral-Implicit and Spectral-Factorized represent using SSGE to perform VI with implicit posteriors and factorized Gaussian posteriors, respectively. Shaded areas represent 3 times standard deviation.}
	\label{fig:bnn-toy}
	\vspace{-0.3cm}
\end{figure}

\vspace{-0.05cm}
\subsubsection{Bayesian Neural Networks}

We evaluate the predictive ability of BNNs with implicit variational posteriors trained by SSGE. To visually access the quality of uncertainty, we choose a 1-D regression problem~\citep{hernandez2015probabilistic, louizos2016structured}. Specifically, $20$ inputs are randomly sampled from $[-4, 4]$, then the target value $y$ is computed with $y=x^3 + \epsilon_n$, $\epsilon_n \sim \mathcal{N}(0, 9)$. We use a BNN with 1 hidden layer and 20 units to model the normalized inputs and targets. We also set the variance of the observation noise to the true value. We compare SSGE with implicit posteriors, Hamiltonian Monte Carlo (HMC) \citep{neal2011mcmc} and Bayes-by-backprop (BBB) \citep{blundell2015weight}. To better demonstrate SSGE's gradient estimation effect, we also test SSGE with a factorized posterior, in comparison to BBB.

We keep 20 chains and run 100k iterations for HMC. All other methods are trained with 100 samples for 20k iterations using Adam optimizer \citep{kingma2014adam}. For SSGE, we set $J=100$. The implicit posteriors we use for weights in both layers are standard normal distributions transformed by fully connected networks with one hidden layer of 100 units. 

As shown in Fig.~\ref{fig:bnn-toy}, HMC, as the golden standard, smoothly fits the training data and outputs sensible uncertainty estimation. 
HMC not only produces large uncertainty outside the data region, its predictive variance also varies even in regions with training points.
This kind of interpolation behavior is hard to be captured by factorized Gaussian posteriors, as shown in Fig.~\ref{fig:bnn-bbb} and \ref{fig:bnn-spectral-factor}.
SSGE with implicit posteriors also has big predictive variances beyond training points, which implies that the BNN trained by SSGE is not overfitting, although it underestimates the uncertainty in the middle region. Also, we observe that SSGE can have similar interpolation behaviors as HMC (see the rightmost two points in Fig.~\ref{fig:bnn-spectral}). 
Besides, SSGE with a factorized posterior has a similar prediction with BBB. Given the network and the variational posterior are both the same, we can attribute this similarity to accurate gradient estimation by SSGE.


\subsubsection{Variational Autoencoders}

From the above example we see that SSGE enables variational posteriors parameterized by implicit distributions. To demonstrate that it scales to larger models and datasets, we adopt the settings in \citet{shi2018kernel} and train a deep convolutional VAE with implicit variational posteriors (\emph{Implicit VAE} for short) on the CelebA dataset. As in their work, the latent dimension is $32$, and the network structure of the decoder is chosen to be the same as DCGAN \citep{radford2015unsupervised}. The observation likelihoods are Gaussian distributions with trainable data-independent variances. The implicit posterior is a deep convolutional net symmetric to the decoder, with Gaussian noises injected into hidden layers. Full details of the model structures can be found in \citet{shi2018kernel}.

To examine how accurate the gradient estimates provided by SSGE are, we conduct experiments under three different settings: a plain VAE with normal variational posteriors, an Implicit VAE trained with the entropy term removed from the ELBO, and an Implicit VAE using SSGE's gradient estimates for the entropy. For SSGE, we set $M=100$, and $\bar{r}=0.99$. In \Cref{fig:celeba-vae,fig:celeba-no-entropy,fig:celeba-ssge} we show samples randomly generated from the trained models. We can see that without the entropy term, the Implicit VAE tends to overfit and produces visually bad generations, while if we retain the entropy term and use SSGE to estimate its gradients, the Implicit VAE can produce realistic samples. 
To quantitatively measure the sample quality, we compare the Fr{\'e}chet Inception Distance (FID)~\citep{heusel2017gans} between real data and random generations from the models. The results are shown in \Cref{fig:fid}. We can see that the Implicit VAE trained by SSGE converges faster and produces samples with slightly better quality than the plain VAE. This is probably due to that implicit posteriors are less likely to overfit~\citep{shi2018kernel}, and SSGE gives accurate gradients for optimizing them. Besides the CelebA experiments, we also tested the models on MNIST dataset and evaluated the test log likelihoods. See \Cref{app:mnist} for details.

\begin{figure*}[t]
    \centering
    \begin{subfigure}[b]{.46\columnwidth}
        \centering
        \centerline{\includegraphics[width=.9\columnwidth]{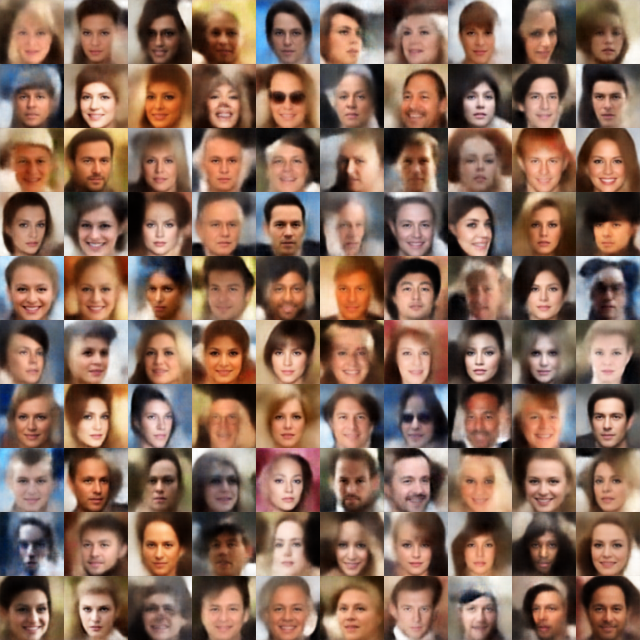}}
        \caption{VAE}
        \label{fig:celeba-vae}
    \end{subfigure}
    \hskip 0.02in
    \begin{subfigure}[b]{.46\columnwidth}
        \centering
        \centerline{\includegraphics[width=.9\columnwidth]{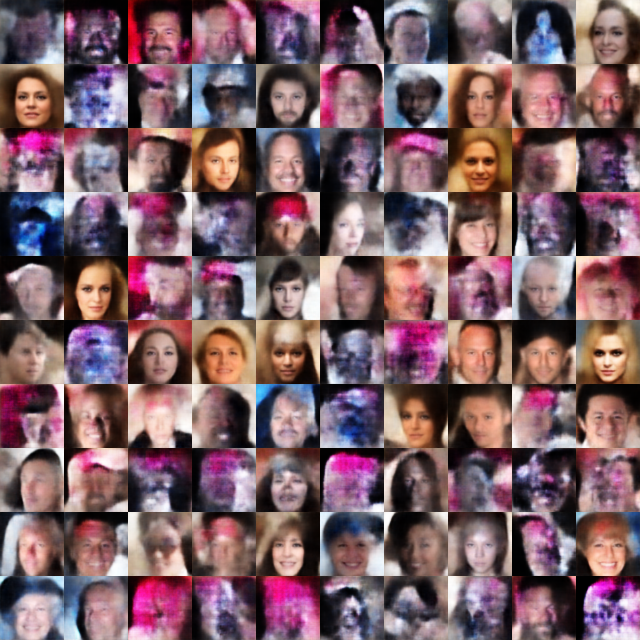}}
        \caption{Implicit VAE, w/o entropy}
        \label{fig:celeba-no-entropy}
    \end{subfigure}
    \hskip 0.02in
    \begin{subfigure}[b]{.46\columnwidth}
        \centering
        \centerline{\includegraphics[width=.9\columnwidth]{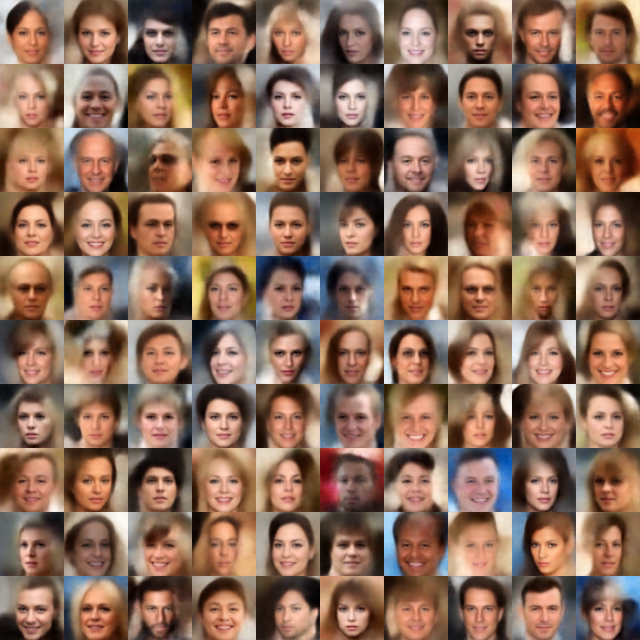}}
        \caption{Implicit VAE, Spectral}
        \label{fig:celeba-ssge}
    \end{subfigure}
    \hskip 0.02in
    \begin{subfigure}[b]{0.64\columnwidth}
        \centering
        \centerline{\includegraphics[height=3.5cm]{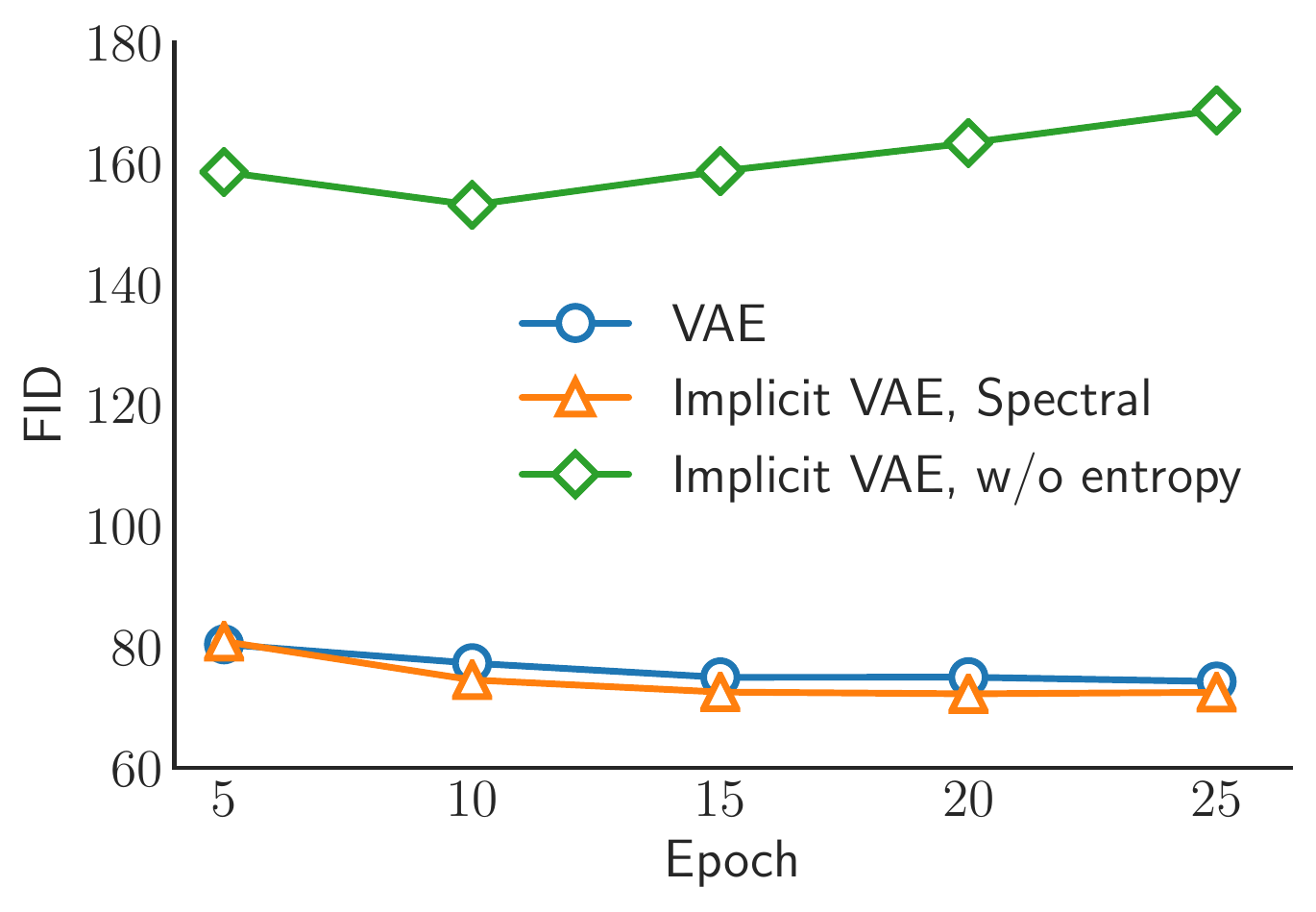}}
        \vspace{-0.15cm}
        \caption{}
        \label{fig:fid}
    \end{subfigure}
    \caption{(a)-(c) CelebA samples generated by VAE, Implicit VAE trained without the entropy term, and Implicit VAE trained by SSGE; (d) Fr{\'e}chet Inception Distances (FID)~\citep{heusel2017gans} between random generated images and real images.}
    \label{fig:celeba}\vspace{-0.2cm}
\end{figure*}


\vspace{-0.15cm}
\section{Discussion}
\textbf{Connection to Kernel PCA}\;
As mentioned in \Cref{sec:nystrom}, the Nystr{\"o}m approximation is closely related to Kernel PCA~\citep{scholkopf1998nonlinear} (KPCA), which is a powerful method for nonlinear dimension reduction. In KPCA, the input data is first projected to a (usually high-dimensional) feature space, where PCA is then applied. The operations in the feature space are handled by the kernel trick. We briefly review the method below.

Given a positive definite kernel $k: \mathcal{X}\times\mathcal{X}\to\mathbb{R}$, we denote the induced RKHS by $\mathcal{H}$ and its corresponding feature map by $\phi: \mathcal{X}\to\mathcal{H}$. Let the data to be analyzed be $\{\bx^i\}_{i=1}^M$, $\bx^i \in \mathcal{X}$. To simplify the derivation, we first assume the data to be centered in the feature space. Then the covariance matrix is formed as
$
\bC = \frac{1}{M}\sum_{i=1}^{M}\phi(\bx^i)\phi(\bx^i)^\top.$
In general, PCA requires to solve the following eigenvalue problem\footnote{We reuse some notations from the above sections (e.g., the eigenvalue $\mu$), and as we shall see, they are closely related.}:
\begin{equation} \label{eq:kpca-eigen}
\bC\bv = \mu\bv.
\end{equation}
A key observation of KPCA is that the eigenvectors lie in the span of the feature vectors, since from \cref{eq:kpca-eigen} we have
\begin{equation} \label{eq:v-by-phi}
\bv = \frac{1}{\mu M}\sum_{i=1}^{M}\left[\phi(\bx^i)^\top\bv\right]\phi(\bx^i) = \sum_{i=1}^{M}\alpha_i\phi(\bx^i).
\end{equation}
Here we use $\balpha = [\alpha_1, \dots, \alpha_M]^\top$ to represent the coefficients. This implies that instead of directly dealing with \cref{eq:kpca-eigen} we can consider a set of $n$ projected equations:
$
\phi(\bx^i)^\top C\bv = \mu\phi(\bx^i)^\top\bv,\; i = 1,\dots,M
$.
Plugging \cref{eq:v-by-phi} here and replacing $\phi(\bx^i)^\top\phi(\bx^j)$ with $k(\bx^i, \bx^j)$ (the kernel trick), we get
$\frac{1}{M} \bK\bK\balpha = \mu\bK\balpha,$
which turns out an eigenvalue problem for $\bK$:
$\frac{1}{M}\bK\balpha = \mu \balpha$. Note that this is exactly the same eigenvalue problem solved in \cref{eq:fredholm-eigen}. As above, we denote by $\bu_1, \dots, \bu_J$ the eigenvectors of $\bK$ that correspond to the $J$ largest eigenvalues $\lambda_1 \geq \dots \geq \lambda_J$, and we have
$\mu_j = \frac{\lambda_j}{M}$.
To determine the $\balpha$s, we set the eigenvectors $\bv$ to have unit lengths: $\bv^\top\bv = \balpha^\top\bK\balpha = \lambda\balpha^\top\balpha = 1.$
So the $\balpha$s should be normalized to have length $\frac{1}{\lambda}$: $
\balpha_j = \frac{1}{\sqrt{\lambda_j}}\bu_j.$ For a new data point $\bx$, KPCA computes the embedding $\xi(\bx)$ (the projection onto the first $J$ eigenvectors) as
$
\xi(\bx) = [\phi(\bx)^\top \bv_1, \dots, \phi(\bx)^\top \bv_J]^\top
= [\balpha_1^\top \bk_{\bx}, \dots, \balpha_J^\top \bk_{\bx}]^\top,
$
where $\bk_{\bx} = [k(\bx, \bx^1), \dots, k(\bx, \bx^M)]^\top$. 
It was pointed out by \citet{williams2000effect} to be equivalent to using the well understood Nystr{\"o}m approximation. We can see this by noticing that each component of $\xi(\bx)$ is identical to \cref{eq:nystrom} up to a scaling factor:
\begin{equation} \label{eq:kpca-nystrom}\vspace{-0.1cm}
\xi_j(\bx) = \balpha_j^\top \bk_{\bx} = \sqrt{\frac{\lambda_j}{M}}\;\hat{\psi}_j(\bx).
\end{equation}
Looking back at \cref{eq:grad-est}, we can see that SSGE estimates the gradients by a linear estimator with KPCA embeddings as input features. As KPCA embeddings are known to automatically adapt to the geometry of the samples, given a suitable kernel is chosen, it can reduce the curse of dimensionality when the estimator is applied to high dimensional spaces, which helps explain the effectiveness of SSGE.

\textbf{Connection to Manifold-modeling Dimension Reduction Methods}\;
It has been pointed out in previous works~\citep{williams2001connection,ham2004kernel,bengio2004learning,bengio2004out} that many successful manifold-modeling dimension reduction methods (e.g., MDS, LLE, Laplacian eigenmaps, and Spectral clustering) can be viewed as KPCA with different ways of constructing data-dependent kernels. We believe it is a promising direction to learn a better kernel from a dataset of samples that could improve the manifold modeling behavior of KPCA embeddings, thus further improving the gradient estimator.

\vspace{-0.15cm}
\section{Conclusion}

We propose the Spectral Stein Gradient Estimator (SSGE) for implicit distributions. Unlike previous methods, SSGE directly estimates the gradient function and thus has a principled out-of-sample extension. Future work may include learning kernels or eigenfunctions in the estimator, as indicated by the error bound as well as the connection to dimension reduction methods.

 \section*{Acknowledgements}

We thank anonymous reviewers for insightful feedbacks, and thank the meta-reviewer and Chang Liu for comments on improving \Cref{thm:stein-identity}. This work was supported by NSFC Projects (Nos. 61620106010, 61621136008, 61332007), Beijing NSF Project (No. L172037), Tiangong Institute for Intelligent Computing, NVIDIA NVAIL Program, Siemens and Intel.


\bibliography{example_paper}

\begin{thebibliography}{52}
\providecommand{\natexlab}[1]{#1}
\providecommand{\url}[1]{\texttt{#1}}
\expandafter\ifx\csname urlstyle\endcsname\relax
  \providecommand{\doi}[1]{doi: #1}\else
  \providecommand{\doi}{doi: \begingroup \urlstyle{rm}\Url}\fi

\bibitem[Andrieu et~al.(2009)Andrieu, Roberts, et~al.]{andrieu2009pseudo}
Andrieu, C., Roberts, G.~O., et~al.
\newblock The pseudo-marginal approach for efficient monte carlo computations.
\newblock \emph{The Annals of Statistics}, 37\penalty0 (2):\penalty0 697--725,
  2009.

\bibitem[Baker(1997)]{baker1997numerical}
Baker, C.~T.
\newblock \emph{The Numerical Treatment of Integral Equations}.
\newblock Clarendon Press, Oxford, 1997.

\bibitem[Belkin \& Niyogi(2003)Belkin and Niyogi]{belkin2003laplacian}
Belkin, M. and Niyogi, P.
\newblock Laplacian eigenmaps for dimensionality reduction and data
  representation.
\newblock \emph{Neural computation}, 15\penalty0 (6):\penalty0 1373--1396,
  2003.

\bibitem[Bengio et~al.(2004{\natexlab{a}})Bengio, Delalleau, Roux, Paiement,
  Vincent, and Ouimet]{bengio2004learning}
Bengio, Y., Delalleau, O., Roux, N.~L., Paiement, J.-F., Vincent, P., and
  Ouimet, M.
\newblock Learning eigenfunctions links spectral embedding and kernel {PCA}.
\newblock \emph{Neural computation}, 16\penalty0 (10):\penalty0 2197--2219,
  2004{\natexlab{a}}.

\bibitem[Bengio et~al.(2004{\natexlab{b}})Bengio, Paiement, Vincent, Delalleau,
  Roux, and Ouimet]{bengio2004out}
Bengio, Y., Paiement, J.-f., Vincent, P., Delalleau, O., Roux, N.~L., and
  Ouimet, M.
\newblock Out-of-sample extensions for {LLE}, {I}somap, {MDS}, eigenmaps, and
  spectral clustering.
\newblock In \emph{Advances in Neural Information Processing Systems}, pp.\
  177--184, 2004{\natexlab{b}}.

\bibitem[Blundell et~al.(2015)Blundell, Cornebise, Kavukcuoglu, and
  Wierstra]{blundell2015weight}
Blundell, C., Cornebise, J., Kavukcuoglu, K., and Wierstra, D.
\newblock Weight uncertainty in neural networks.
\newblock In \emph{International Conference on Machine Learning}, pp.\
  1613--1622, 2015.

\bibitem[Borg \& Groenen(2005)Borg and Groenen]{borg2005modern}
Borg, I. and Groenen, P.~J.
\newblock \emph{Modern multidimensional scaling: Theory and applications}.
\newblock Springer Science \& Business Media, 2005.

\bibitem[Burges et~al.(2010)]{burges2010dimension}
Burges, C.~J. et~al.
\newblock Dimension reduction: A guided tour.
\newblock \emph{Foundations and Trends{\textregistered} in Machine Learning},
  2\penalty0 (4):\penalty0 275--365, 2010.

\bibitem[Chwialkowski et~al.(2016)Chwialkowski, Strathmann, and
  Gretton]{chwialkowski2016kernel}
Chwialkowski, K., Strathmann, H., and Gretton, A.
\newblock A kernel test of goodness of fit.
\newblock In \emph{International Conference on Machine Learning}, pp.\
  2606--2615, 2016.

\bibitem[Dai et~al.(2016)Dai, He, Dai, and Song]{dai2015provable}
Dai, B., He, N., Dai, H., and Song, L.
\newblock Provable bayesian inference via particle mirror descent.
\newblock In \emph{International Conference on Artificial Intelligence and
  Statistics}, pp.\  985--994, 2016.

\bibitem[Donahue et~al.(2016)Donahue, Kr{\"a}henb{\"u}hl, and
  Darrell]{donahue2016adversarial}
Donahue, J., Kr{\"a}henb{\"u}hl, P., and Darrell, T.
\newblock Adversarial feature learning.
\newblock \emph{arXiv preprint arXiv:1605.09782}, 2016.

\bibitem[Dumoulin et~al.(2016)Dumoulin, Belghazi, Poole, Lamb, Arjovsky,
  Mastropietro, and Courville]{dumoulin2016adversarially}
Dumoulin, V., Belghazi, I., Poole, B., Lamb, A., Arjovsky, M., Mastropietro,
  O., and Courville, A.
\newblock Adversarially learned inference.
\newblock \emph{arXiv preprint arXiv:1606.00704}, 2016.

\bibitem[Goodfellow et~al.(2014)Goodfellow, Pouget-Abadie, Mirza, Xu,
  Warde-Farley, Ozair, Courville, and Bengio]{goodfellow2014generative}
Goodfellow, I., Pouget-Abadie, J., Mirza, M., Xu, B., Warde-Farley, D., Ozair,
  S., Courville, A., and Bengio, Y.
\newblock Generative adversarial nets.
\newblock In \emph{Advances in Neural Information Processing Systems}, pp.\
  2672--2680, 2014.

\bibitem[Gorham \& Mackey(2015)Gorham and Mackey]{gorham2015measuring}
Gorham, J. and Mackey, L.
\newblock Measuring sample quality with stein's method.
\newblock In \emph{Advances in Neural Information Processing Systems}, pp.\
  226--234, 2015.

\bibitem[Ham et~al.(2004)Ham, Lee, Mika, and Sch{\"o}lkopf]{ham2004kernel}
Ham, J., Lee, D.~D., Mika, S., and Sch{\"o}lkopf, B.
\newblock A kernel view of the dimensionality reduction of manifolds.
\newblock In \emph{International Conference on Machine Learning}, pp.\ ~47,
  2004.

\bibitem[Hern{\'a}ndez-Lobato \& Adams(2015)Hern{\'a}ndez-Lobato and
  Adams]{hernandez2015probabilistic}
Hern{\'a}ndez-Lobato, J.~M. and Adams, R.
\newblock Probabilistic backpropagation for scalable learning of bayesian
  neural networks.
\newblock In \emph{International Conference on Machine Learning}, pp.\
  1861--1869, 2015.

\bibitem[Heusel et~al.(2017)Heusel, Ramsauer, Unterthiner, Nessler, Klambauer,
  and Hochreiter]{heusel2017gans}
Heusel, M., Ramsauer, H., Unterthiner, T., Nessler, B., Klambauer, G., and
  Hochreiter, S.
\newblock Gans trained by a two time-scale update rule converge to a nash
  equilibrium.
\newblock \emph{arXiv preprint arXiv:1706.08500}, 2017.

\bibitem[Husz{\'a}r(2017)]{huszar2017variational}
Husz{\'a}r, F.
\newblock Variational inference using implicit distributions.
\newblock \emph{arXiv preprint arXiv:1702.08235}, 2017.

\bibitem[Hyv{\"a}rinen(2005)]{hyvarinen2005estimation}
Hyv{\"a}rinen, A.
\newblock Estimation of non-normalized statistical models by score matching.
\newblock \emph{Journal of Machine Learning Research}, 6\penalty0
  (Apr):\penalty0 695--709, 2005.

\bibitem[Izbicki et~al.(2014)Izbicki, Lee, and Schafer]{izbicki2014high}
Izbicki, R., Lee, A., and Schafer, C.
\newblock High-dimensional density ratio estimation with extensions to
  approximate likelihood computation.
\newblock In \emph{International Conference on Artificial Intelligence and
  Statistics}, pp.\  420--429, 2014.

\bibitem[Kingma \& Ba(2014)Kingma and Ba]{kingma2014adam}
Kingma, D.~P. and Ba, J.
\newblock Adam: A method for stochastic optimization.
\newblock \emph{arXiv preprint arXiv:1412.6980}, 2014.

\bibitem[Kingma \& Welling(2013)Kingma and Welling]{kingma2013auto}
Kingma, D.~P. and Welling, M.
\newblock Auto-encoding variational bayes.
\newblock \emph{arXiv preprint arXiv:1312.6114}, 2013.

\bibitem[Li \& Turner(2018)Li and Turner]{li2018gradient}
Li, Y. and Turner, R.~E.
\newblock Gradient estimators for implicit models.
\newblock In \emph{International Conference on Learning Representations}, 2018.

\bibitem[Liu \& Feng(2016)Liu and Feng]{liu2016two}
Liu, Q. and Feng, Y.
\newblock Two methods for wild variational inference.
\newblock \emph{arXiv preprint arXiv:1612.00081}, 2016.

\bibitem[Liu \& Wang(2016)Liu and Wang]{liu2016stein}
Liu, Q. and Wang, D.
\newblock Stein variational gradient descent: A general purpose bayesian
  inference algorithm.
\newblock In \emph{Advances In Neural Information Processing Systems}, pp.\
  2370--2378, 2016.

\bibitem[Liu et~al.(2016)Liu, Lee, and Jordan]{liu2016kernelized}
Liu, Q., Lee, J., and Jordan, M.
\newblock A kernelized stein discrepancy for goodness-of-fit tests.
\newblock In \emph{International Conference on Machine Learning}, pp.\
  276--284, 2016.

\bibitem[Louizos \& Welling(2016)Louizos and Welling]{louizos2016structured}
Louizos, C. and Welling, M.
\newblock Structured and efficient variational deep learning with matrix
  gaussian posteriors.
\newblock In \emph{International Conference on Machine Learning}, pp.\
  1708--1716, 2016.

\bibitem[Mescheder et~al.(2017)Mescheder, Nowozin, and
  Geiger]{mescheder2017adversarial}
Mescheder, L., Nowozin, S., and Geiger, A.
\newblock Adversarial variational bayes: Unifying variational autoencoders and
  generative adversarial networks.
\newblock In \emph{International Conference on Machine Learning}, pp.\
  2391--2400, 2017.

\bibitem[Mohamed \& Lakshminarayanan(2016)Mohamed and
  Lakshminarayanan]{mohamed2016learning}
Mohamed, S. and Lakshminarayanan, B.
\newblock Learning in implicit generative models.
\newblock \emph{arXiv preprint arXiv:1610.03483}, 2016.

\bibitem[Murray \& Adams(2010)Murray and Adams]{murray2010slice}
Murray, I. and Adams, R.~P.
\newblock Slice sampling covariance hyperparameters of latent gaussian models.
\newblock In \emph{Advances in Neural Information Processing Systems}, pp.\
  1732--1740, 2010.

\bibitem[Neal et~al.(2011)]{neal2011mcmc}
Neal, R.~M. et~al.
\newblock Mcmc using hamiltonian dynamics.
\newblock \emph{Handbook of Markov Chain Monte Carlo}, 2\penalty0 (11), 2011.

\bibitem[Nystr{\"o}m(1930)]{nystrom1930praktische}
Nystr{\"o}m, E.~J.
\newblock {\"U}ber die praktische aufl{\"o}sung von integralgleichungen mit
  anwendungen auf randwertaufgaben.
\newblock \emph{Acta Mathematica}, 54\penalty0 (1):\penalty0 185--204, 1930.

\bibitem[Parlett \& Scott(1979)Parlett and Scott]{parlett1979lanczos}
Parlett, B.~N. and Scott, D.~S.
\newblock The lanczos algorithm with selective orthogonalization.
\newblock \emph{Mathematics of computation}, 33\penalty0 (145):\penalty0
  217--238, 1979.

\bibitem[Radford et~al.(2015)Radford, Metz, and
  Chintala]{radford2015unsupervised}
Radford, A., Metz, L., and Chintala, S.
\newblock Unsupervised representation learning with deep convolutional
  generative adversarial networks.
\newblock \emph{arXiv preprint arXiv:1511.06434}, 2015.

\bibitem[Ranganath et~al.(2016)Ranganath, Tran, Altosaar, and
  Blei]{ranganath2016operator}
Ranganath, R., Tran, D., Altosaar, J., and Blei, D.
\newblock Operator variational inference.
\newblock In \emph{Advances in Neural Information Processing Systems}, pp.\
  496--504, 2016.

\bibitem[Rosasco et~al.(2010)Rosasco, Belkin, and Vito]{rosasco2010learning}
Rosasco, L., Belkin, M., and Vito, E.~D.
\newblock On learning with integral operators.
\newblock \emph{Journal of Machine Learning Research}, 11\penalty0
  (Feb):\penalty0 905--934, 2010.

\bibitem[Roweis \& Saul(2000)Roweis and Saul]{roweis2000nonlinear}
Roweis, S.~T. and Saul, L.~K.
\newblock Nonlinear dimensionality reduction by locally linear embedding.
\newblock \emph{Science}, 290\penalty0 (5500):\penalty0 2323--2326, 2000.

\bibitem[Sch{\"o}lkopf et~al.(1998)Sch{\"o}lkopf, Smola, and
  M{\"u}ller]{scholkopf1998nonlinear}
Sch{\"o}lkopf, B., Smola, A., and M{\"u}ller, K.-R.
\newblock Nonlinear component analysis as a kernel eigenvalue problem.
\newblock \emph{Neural computation}, 10\penalty0 (5):\penalty0 1299--1319,
  1998.

\bibitem[Sejdinovic \& Gretton()Sejdinovic and Gretton]{sejdinovic2012rkhs}
Sejdinovic, D. and Gretton, A.
\newblock What is an rkhs?

\bibitem[Sejdinovic et~al.(2014)Sejdinovic, Strathmann, Garcia, Andrieu, and
  Gretton]{sejdinovic2014kernel}
Sejdinovic, D., Strathmann, H., Garcia, M.~L., Andrieu, C., and Gretton, A.
\newblock Kernel adaptive metropolis-hastings.
\newblock In \emph{International Conference on Machine Learning}, pp.\
  1665--1673, 2014.

\bibitem[Shi et~al.(2017)Shi, Chen, Zhu, Sun, Luo, Gu, and Zhou]{zhusuan2017}
Shi, J., Chen, J., Zhu, J., Sun, S., Luo, Y., Gu, Y., and Zhou, Y.
\newblock Zhu{S}uan: A library for {B}ayesian deep learning.
\newblock \emph{arXiv preprint arXiv:1709.05870}, 2017.

\bibitem[Shi et~al.(2018)Shi, Sun, and Zhu]{shi2018kernel}
Shi, J., Sun, S., and Zhu, J.
\newblock Kernel implicit variational inference.
\newblock In \emph{International Conference on Learning Representations}, 2018.

\bibitem[Sinha \& Belkin(2009)Sinha and Belkin]{sinha2009semi}
Sinha, K. and Belkin, M.
\newblock Semi-supervised learning using sparse eigenfunction bases.
\newblock In \emph{Advances in Neural Information Processing Systems}, pp.\
  1687--1695, 2009.

\bibitem[Sriperumbudur et~al.(2017)Sriperumbudur, Fukumizu, Gretton,
  Hyv{\"a}rinen, and Kumar]{sriperumbudur2017density}
Sriperumbudur, B., Fukumizu, K., Gretton, A., Hyv{\"a}rinen, A., and Kumar, R.
\newblock Density estimation in infinite dimensional exponential families.
\newblock \emph{The Journal of Machine Learning Research}, 18\penalty0
  (1):\penalty0 1830--1888, 2017.

\bibitem[Stein(1981)]{stein1981estimation}
Stein, C.~M.
\newblock Estimation of the mean of a multivariate normal distribution.
\newblock \emph{The Annals of Statistics}, pp.\  1135--1151, 1981.

\bibitem[Strathmann et~al.(2015)Strathmann, Sejdinovic, Livingstone, Szabo, and
  Gretton]{strathmann2015gradient}
Strathmann, H., Sejdinovic, D., Livingstone, S., Szabo, Z., and Gretton, A.
\newblock Gradient-free hamiltonian monte carlo with efficient kernel
  exponential families.
\newblock In \emph{Advances in Neural Information Processing Systems}, pp.\
  955--963, 2015.

\bibitem[Sutherland et~al.(2018)Sutherland, Strathmann, Arbel, and
  Gretton]{pmlr-v84-sutherland18a}
Sutherland, D., Strathmann, H., Arbel, M., and Gretton, A.
\newblock Efficient and principled score estimation with nyström kernel
  exponential families.
\newblock In \emph{International Conference on Artificial Intelligence and
  Statistics}, pp.\  652--660, 2018.

\bibitem[Tran et~al.(2017)Tran, Ranganath, and Blei]{tran2017hierarchical}
Tran, D., Ranganath, R., and Blei, D.
\newblock Hierarchical implicit models and likelihood-free variational
  inference.
\newblock In \emph{Advances in Neural Information Processing Systems}, pp.\
  5529--5539, 2017.

\bibitem[Weiss(1999)]{weiss1999segmentation}
Weiss, Y.
\newblock Segmentation using eigenvectors: a unifying view.
\newblock In \emph{The proceedings of the seventh IEEE international conference
  on computer vision}, volume~2, pp.\  975--982. IEEE, 1999.

\bibitem[Williams \& Seeger(2000)Williams and Seeger]{williams2000effect}
Williams, C. and Seeger, M.
\newblock The effect of the input density distribution on kernel-based
  classifiers.
\newblock In \emph{International Conference on Machine Learning}, 2000.

\bibitem[Williams(2001)]{williams2001connection}
Williams, C.~K.
\newblock On a connection between kernel {PCA} and metric multidimensional
  scaling.
\newblock In \emph{Advances in Neural Information Processing Systems}, pp.\
  675--681, 2001.

\bibitem[Williams \& Seeger(2001)Williams and Seeger]{williams2001using}
Williams, C.~K. and Seeger, M.
\newblock Using the {N}ystr{\"o}m method to speed up kernel machines.
\newblock In \emph{Advances in Neural Information Processing Systems}, pp.\
  682--688, 2001.

\end{thebibliography}
\bibliographystyle{icml2018}

\onecolumn
\clearpage

\appendix

\icmltitle{A Spectral Approach to Gradient Estimation for Implicit Distributions:\\
Appendix}

\section{Proof of Proposition 1} \label{app:proof-eigen-stein}

We first introduce the following lemmas.

\begin{lem}[\citealt{liu2016kernelized}, Proposition 3.5] \label{lem:rkhs-stein}
	Let $\mathcal{H}$ denote the Reproducing Kernel Hilbert Space (RKHS) induced by kernel $k$. If $k(\cdot, \cdot)$ has continuous second order partial derivatives, and both $k(\bx, \cdot)$ and $k(\cdot, \bx)$ are in the Stein class of $q$ for any fixed $\bx$, then $\forall f\in \mathcal{H}$, $f$ is in the Stein class of $q$.
\end{lem}

\begin{lem}[Mercer's theorem] \label{lem:mercer}
	Let $k$ be a continuous kernel on compact metric space $\mathcal{X}$. $q$ is a finite Borel measure on $\mathcal{X}$. Then for $\{\psi_j\}_{j\geq 1}$ that satisfy \cref{eq:fredholm}, $\forall \bx,\by \in \mathcal{X}$:
	\begin{equation*}
	k(\bx, \by) = \sum_{j}\mu_j\psi_j(\bx)\psi_j(\by).
	\end{equation*}
\end{lem}

\begin{proof}
	See \citet{sejdinovic2012rkhs}, Theorem 50.
\end{proof}

\begin{lem}[\citealt{sejdinovic2012rkhs}, Theorem 51] \label{lem:l2-vs-rkhs}
	Let $\mathcal{X}$ be a compact metric space and k: $\mathcal{X}\times\mathcal{X}\to\mathbb{R}$ be a continuous kernel, Define:
	\begin{equation*}
		\mathcal{H} = \left\{f=\sum_{i}a_i\psi_i:\left\{\frac{a_i}{\sqrt{\mu_i}}\right\}\in \ell^2 \right\}.
	\end{equation*}
	Then $\mathcal{H}$ is the same space as the RKHS induced by $k$.
\end{lem}

Then we prove \Cref{prop:eigen-boundary}.

\begin{proof}
	In \Cref{lem:l2-vs-rkhs} we set $a_j=1$, $a_i=0\;(\forall i\neq j)$, then we have $\psi_j \in \mathcal{H}$. According to \Cref{lem:rkhs-stein}, we can conclude that $\psi_j$ is in the Stein class of $q$	.
\end{proof}

\section{Error Bound of SSGE}
\label{app:theory}

Define
\begin{equation}
	g_i(\bx) = \sum_{j=1}^\infty \beta_{ij}\psi_j(\bx), \quad
	g_{i, J}(\bx) = \sum_{j=1}^J \beta_{ij}\psi_j(\bx), \quad
	\tilde{g}_{i, J}(\bx) = \sum_{j=1}^J \beta_{ij}\hat{\psi}_j(\bx), \quad
	\hat{g}_{i, J}(\bx) = \sum_{j=1}^J \hat{\beta}_{ij}\hat{\psi}_j(\bx),
\end{equation}
which correspond to the major approximations in each step.

%
%
%
%

\begin{lem}[\citealt{sinha2009semi, izbicki2014high, rosasco2010learning}] \label{lem:nystrom-vs-eigen}
	For all $1\leq j\leq J$,
	\begin{equation}
		\int \left(\hat{\psi}_j(\bx) - \psi_j(\bx)\right)^2\;dq=O_p\left(\frac{1}{\mu_j\delta_j^2M}\right),
	\end{equation}
	where $\delta_j = \mu_j - \mu_{j +1}$.
\end{lem}

\begin{lem}[\citealt{sinha2009semi, izbicki2014high}] \label{lem:nystrom-inner}
	For all $1\leq j\leq J$,
	\begin{equation}
	\int \hat{\psi}_j(\bx)^2\;dq=O_p\left(\frac{1}{\mu_j\Delta_J^2M}\right) + 1,
	\end{equation}
	and for all $1\leq i\leq J, i\neq j$,
	\begin{equation}
	\int \hat{\psi}_i(\bx)\hat{\psi}_j(\bx)\;dq=O_p\left(\left(\frac{1}{\sqrt{\mu_i}}+ \frac{1}{\sqrt{\mu_j}}\right)\frac{1}{\Delta_J\sqrt{M}}\right),
	\end{equation}
	where $\Delta_J = \min_{1\leq j\leq J} \delta_j$.
\end{lem}

\begin{lem} \label{lem:series-nystrom-vs-true}
	\begin{equation}
	\int \left|\tilde{g}_{i, J}(\bx) - g_{i, J}(\bx)\right|^2\;dq = JO_p\left(\frac{C}{\mu_J\Delta_J^2M}\right).
	\end{equation}
\end{lem}
\begin{proof}
	By Cauchy-Schwartz inequality, \Cref{asp:square-integrable} and \Cref{lem:nystrom-vs-eigen}:
	\begin{equation*}
		\begin{aligned}
		\int \left|\tilde{g}_{i, J}(\bx) - g_{i, J}(\bx)\right|^2\;dq &= \int \left|\sum_{j=1}^J\beta_{ij}\left(\psi_j(\bx) - \hat{\psi}_j(\bx)\right)\right|^2\;dq \\
		&\leq \left(\sum_{j=1}^J\beta_{ij}^2\right)\left(\sum_{j=1}^J\int\left(\psi_j(\bx) - \hat{\psi}_j(\bx)\right)^2\;dq\right)\\
		&= JO_p\left(\frac{C}{\mu_J\Delta_J^2M}\right).
		\end{aligned}
	\end{equation*}
\end{proof}

\begin{lem} \label{lem:grad-nystrom-vs-grad-eigen}
	For all $1\leq j\leq J$,
	\begin{equation}
		\left(\int \left(\nabla_{x_i}\psi_j(\bx) - \nabla_{x_i}\hat{\psi}_j(\bx)\right)\;dq\right)^2 = O_p\left(\frac{C}{\mu_j\delta_j^2M}\right).
	\end{equation}
	
\end{lem}

\begin{proof}
	Denote $\delta(\bx) = \psi_j(\bx) -  \hat{\psi}_j(\bx) $.
	According to \Cref{asp:kernel-boundary}, it is easy to see that $\hat{\psi}_j(\bx)$ satisfies the boundary condition:
	\begin{equation}
		\int \nabla_{\bx}[\hat{\psi}_j(\bx)q(\bx)] d\bx = \bzero.
	\end{equation}
	And from the proof of \Cref{prop:eigen-boundary}, we know $\psi_j(\bx)$ satisfies the boundary condition. Combining the two, we have:
	\begin{equation} \label{eq:diff-boundary}
	\int \nabla_{\bx} [\delta(\bx) q(\bx)]\;d\bx = \bzero.
	\end{equation}
	By \cref{eq:diff-boundary}, \Cref{lem:nystrom-vs-eigen} and \Cref{asp:square-integrable}, we have
	\begin{equation}
	\begin{aligned}
	\left(\int \nabla_{x_i} \delta(\bx)dq\right)^2 &= \left(\int \nabla_{x_i} [\delta(\bx) q(\bx)] - \delta(\bx) \nabla_{x_i} q(\bx)\;d\bx\right)^2 \\
	&= \left(\int  \delta(\bx) \nabla_{x_i} \log q(\bx)\;dq\right)^2 \\
	&\le \left(\int  \delta(\bx)^2\;dq\right)\left(\int g_i(\bx)^2\;dq\right) \\
	&= O_p\left(\frac{C}{\mu_j\delta_j^2M}\right).
	\end{aligned}
	\end{equation}
\end{proof}

\begin{lem} \label{lem:coeff-bound}
	For all $1\leq j\leq J$,
	\begin{equation}
		(\beta_{ij} - \hat{\beta}_{ij})^2 = O_p\left(\frac{1}{M}\right) +  O_p\left(\frac{C}{\mu_j\delta_j^2M}\right).
	\end{equation}
\end{lem}

\begin{proof}
	
	
	\begin{equation}
	\begin{aligned}
	    \frac{1}{2}(\beta_{ij} - \hat{\beta}_{ij})^2 &\le \left(\beta_{ij} - \frac{1}{M}\sum_{m=1}^M \nabla_{x_i} \psi_j(\bx^m)\right)^2 + \left( \frac{1}{M}\sum_{m=1}^M \left(\nabla_{x_i} \psi_j(\bx^m) - \nabla_{x_i} \hat{\psi}_j(\bx^m)\right)\right)^2 \\
	    &\le   O_p\left(\frac{1}{M}\right) + 2\left[ \frac{1}{M}\sum_{m=1}^M \left(\nabla_{x_i} \psi_j(\bx^m) - \nabla_{x_i} \hat{\psi}_j(\bx^m)\right) - \int \left(\nabla_{x_i} \psi_j(\bx) - \nabla_{x_i} \hat{\psi}_j(\bx) \right)\;dq\right]^2 \\
	    & \quad + 2\left[\int \left(\nabla_{x_i} \psi_j(\bx) - \nabla_{x_i} \hat{\psi}_j(\bx) \right)\;dq\right]^2 \\
	    &= O_p\left(\frac{1}{M}\right) + 2 O_p\left(\frac{1}{M}\right) + 2\left(\int\left(\nabla_{x_i} \psi_j(\bx) - \nabla_{x_i} \hat{\psi}_j(\bx)\right)dq\right)^2.
	\end{aligned}
	\end{equation}
Therefore, by \Cref{lem:grad-nystrom-vs-grad-eigen} we have
\begin{equation}
    (\beta_{ij} - \hat{\beta}_{ij})^2 = O_p\left(\frac{1}{M}\right) +  O_p\left(\frac{C}{\mu_j\delta_j^2M}\right).
\end{equation}

\end{proof}

\begin{lem} \label{lem:series-coeff-vs}
	\begin{equation}
	\int \left|\tilde{g}_{i, J}(\bx) - \hat{g}_{i, J} (\bx)\right|^2 dq = J^2 \left(O_p\left(\frac{1}{M}\right) +  O_p\left(\frac{C}{\mu_j\Delta_j^2M}\right)\right)
	\end{equation}
\end{lem}

\begin{proof}
	By applying Minkowski inequality, Cauchy-Schwartz inequality, \Cref{lem:coeff-bound} and \Cref{lem:nystrom-inner}, we have
	\begin{equation}
	\begin{aligned}
	    &\int \left|\tilde{g}_{i, J}(\bx) - \hat{g}_{i, J} (\bx)\right|^2 dq = \int \left|\sum_{j=1}^J \beta_{ij}\hat{\psi}_j(\bx)  - \sum_{j=1}^J \hat{\beta}_{ij}\hat{\psi}_j(\bx)\right|^2 dq = \int \left|\sum_{j=1}^J (\beta_{ij} - \hat{\beta}_{ij})\hat{\psi}_j(\bx) \right|^2 dq \\
	    &\le \left\{\sum_{j=1}^J \left[\int \left| (\beta_{ij} - \hat{\beta}_{ij})\hat{\psi}_j(\bx)  \right|^2dq\right]^{\frac{1}{2}}\right\}^2
	    \le  \left\{\sum_{j=1}^J \left[\int \left| (\beta_{ij} - \hat{\beta}_{ij}) \right|^2dq \int \hat{\psi}_j^2(\bx) \;dq \right]^{\frac{1}{2}}\right\}^2 \\
		&= \left\{\sum_{j=1}^J \left[O_p\left(\frac{1}{M}\right) +  O_p\left(\frac{C}{\mu_j\delta_j^2M}\right)\right]^{\frac{1}{2}}\left[O_p\left(\frac{1}{\mu_j\Delta_J^2M}\right) + 1\right]^{\frac{1}{2}}\right\}^2 \\
	    &= J^2 \left(O_p\left(\frac{1}{M}\right) +  O_p\left(\frac{C}{\mu_J\Delta_J^2M}\right)\right)
	\end{aligned}
	\end{equation}
\end{proof}

\begin{thm}[Estimation Error] \label{lem:est-error}
	\begin{equation}
		\int \left|\hat{g}_{i, J}(\bx) - g_{i, J}(\bx)\right|^2\;dq = J^2 \left(O_p\left(\frac{1}{M}\right) +  O_p\left(\frac{C}{\mu_J\Delta_J^2M}\right)\right) + JO_p\left(\frac{C}{\mu_J\Delta_J^2M}\right)
	\end{equation}
\end{thm}

\begin{proof}
	By \cref{lem:series-nystrom-vs-true} and \cref{lem:series-coeff-vs}.
	\begin{equation}
		\begin{aligned}
			\int \left|\hat{g}_{i,J}(\bx) - g_{i,J}(\bx)\right|^2\;dq &\leq \int \left|\tilde{g}_{i,J}(\bx) - g_{i,J}(\bx)\right|^2\;dq + \int \left|\tilde{g}_{i,J}(\bx) - \hat{g}_{i,J} (\bx)\right|^2 dq \\
			&= J^2 \left(O_p\left(\frac{1}{M}\right) +  O_p\left(\frac{C}{\mu_J\Delta_J^2M}\right)\right) + JO_p\left(\frac{C}{\mu_J\Delta_J^2M}\right)
		\end{aligned}
	\end{equation}
\end{proof}

\begin{thm}[Approximation Error]\label{lem:trunc-error}
	\begin{equation}
		\int \left|g_{i, J}(\bx) - g_i(\bx)\right|^2 dq = \|g_i\|^2_{\mathcal{H}}O(\mu_J)
	\end{equation}
\end{thm}

\begin{proof}
	\begin{equation}
		\int \left|g_{i, J}(\bx) - g_i(\bx)\right|^2 dq = \sum_{j> J}\beta_{ij}^2 = \sum_{j>J} \frac{\beta_{ij}^2}{\mu_j}\mu_j \leq \mu_J\sum_{j>J}\frac{\beta_{ij}^2}{\mu_j} = \mu_J\|g_i\|^2_\mathcal{H}
	\end{equation}
\end{proof}

\begin{thm}[Error Bound of SSGE]
	\begin{equation}
		\int \left(\hat{g}_{i, J}(\bx) - g_i(\bx)\right)^2\;dq = J^2 \left(O_p\left(\frac{1}{M}\right) +  O_p\left(\frac{C}{\mu_J\Delta_J^2M}\right)\right) + JO_p\left(\frac{C}{\mu_J\Delta_J^2M}\right) + \|g_i\|^2_{\mathcal{H}}O(\mu_J).
	\end{equation}
\end{thm}

\begin{proof}
	By \cref{lem:est-error} and \cref{lem:trunc-error}, we have
	\begin{equation}
	\begin{aligned}
	\int \left(\hat{g}_{i, J}(\bx) - g_i(\bx)\right)^2\;dq &\leq \int \left|\hat{g}_{i,J}(\bx) - g_{i, J}(\bx)\right|^2\;dq + \int \left|g_{i, J}(\bx) - g_i(\bx)\right|^2 dq \\
	&= J^2 \left(O_p\left(\frac{1}{M}\right) +  O_p\left(\frac{C}{\mu_J\Delta_J^2M}\right)\right) + JO_p\left(\frac{C}{\mu_J\Delta_J^2M}\right) + \|g_i\|^2_{\mathcal{H}}O(\mu_J)
	\end{aligned}
	\end{equation}
\end{proof}

\section{Derivation of Gradient Estimates for Entropy}
First, we decompose the gradients into two terms:
\label{app:entropy-grad}
\begin{equation*}
\begin{aligned}
\nabla_{\phi}\mathbb{H}(q) = -\nabla_{\phi}\mathbb{E}_{q_{\phi}}\log q_{\phi}(\bx) 
= -\nabla_{\phi}\mathbb{E}_q \log q_{\phi}(\bx) - \nabla_{\phi}\mathbb{E}_{q_{\phi}} \log q(\bx).
\end{aligned}
\end{equation*}
Then it is easy to see that the first term is zero:
\begin{equation*}
    \nabla_{\phi}\mathbb{E}_q \log q_{\phi}(\bx) = \int q(\bx)\nabla_{\phi}\log q_{\phi}(\bx) \;d\bx
    = \int \nabla_{\phi}q_{\phi}(\bx)\;d\bx
    = \nabla_{\phi}\int q_{\phi}(\bx)\;d\bx = 0.
\end{equation*}
So we have
\begin{equation} \label{eq:entropy-second}
    \nabla_{\phi}\mathbb{H}(q) = -\nabla_{\phi}\mathbb{E}_{q_{\phi}}\log q(\bx).
\end{equation}
A low-variance Monte-Carlo estimate of \cref{eq:entropy-second} can	 be obtained when $q(\bx)$ is continuous and reparameterizable~\citep{kingma2013auto}. For example, if there exists a random variable $\beps\sim \mathcal{N}(\bzero, \bI)$, so that $\bx = f(\beps; \phi)$ follows the same distribution as $q(\bx)$, we have
\begin{equation*}
\begin{aligned}
\nabla_{\phi}\mathbb{H}(q) &= -\nabla_{\phi}\mathbb{E}_{\beps\sim \mathcal{N}(\bzero, \bI)}\log q(f(\beps;\phi)) \\
&= -\mathbb{E}_{\beps\sim \mathcal{N}(\bzero, \bI)}\nabla_{\phi}\log q(f(\beps;\phi)) \\
&= -\mathbb{E}_{\beps\sim \mathcal{N}(\bzero, \bI)}\nabla_{\bx}\log q(f(\beps;\phi))\nabla_{\phi}f(\beps;\phi),
\end{aligned}
\end{equation*}
where the intractable term $\nabla_{\bx}\log q(f(\beps;\phi))$ can be easily estimated by SSGE.

\section{MNIST Results}
\label{app:mnist}

We did an MNIST experiment on a VAE with an 8-dim latent space. In \Cref{fig:mnist-vae,fig:mnist-noent,fig:mnist-spectral} we plot random generations by a plain VAE, an Implicit VAE with the entropy term removed, and an Implicit VAE trained by SSGE, all with the same decoder structures. In \Cref{tab:mnist} we show the log likelihoods of the trained models (VAE and the Implicit VAE with SSGE) on 2048 test images using Annealed Importance Sampling (AIS). The results are averaged over 10 runs. For reference, we also include the results of 8-dim VAEs from the AVB paper~\citep{mescheder2017adversarial}. Though their decoder structure is not the same as ours (the structure is even not the same for their three models), we can see our method is slightly better than plain VAE without other tricks, while AVB has to rely on Adaptive Contrast (AC) and different decoder structures.

\begin{figure}[h]
    \begin{subfigure}[b]{.23\columnwidth}
        \centering
        \centerline{\includegraphics[height=3.5cm]{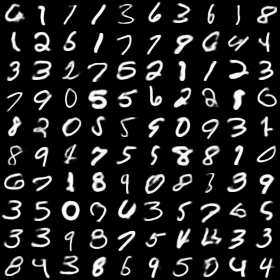}}
        \caption{VAE}
        \label{fig:mnist-vae} 
    \end{subfigure}
    \hskip 0.02in
    \begin{subfigure}[b]{.23\columnwidth}
        \centering
        \centerline{\includegraphics[height=3.5cm]{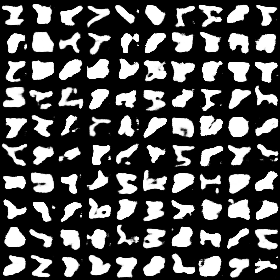}}
        \caption{Implicit VAE, w/o entropy}
        \label{fig:mnist-noent} 
    \end{subfigure}
    \hskip 0.02in
    \begin{subfigure}[b]{.23\columnwidth}
        \centering
        \centerline{\includegraphics[height=3.5cm]{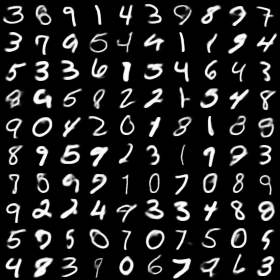}}
        \caption{Implicit VAE, Spectral}
        \label{fig:mnist-spectral} 
    \end{subfigure}
    \hskip 0.02in
    \begin{subtable}[b]{0.28\columnwidth}
        \centering
        \begin{sc}                
            {\footnotesize
                \begin{tabular}{lr}
                    \toprule
                    Method & Test LL. \\
                    \midrule
                    VAE  		  & -89.5 $\pm$ 0.6  \\
                    Spectral 	  & -89.4 $\pm$ 0.7  \\
                    \midrule
                    VAE & -90.9 $\pm$ 0.6 \\
                    AVB  & -91.2 $\pm$ 0.6 \\
                    AVB + AC & -89.6 $\pm$ 0.6  \\
                    \bottomrule
            \end{tabular}}
        \end{sc}
        \caption{}
        \label{tab:mnist}
    \end{subtable}
    \vspace{-0.2cm}
    \caption{Results on the MNIST dataset: (a)-(c) Samples generated by VAE, Implicit VAE trained without the entropy term, and Implicit VAE trained by SSGE; (d) Test log likelihoods evaluated by AIS.}\vspace{-0.2cm}
\end{figure}

\end{document}